\documentclass{article}

\usepackage{arxiv}

\usepackage[utf8]{inputenc}
\usepackage[T1]{fontenc}
\usepackage{hyperref}
\usepackage{url}
\usepackage{booktabs}
\usepackage{amsfonts}
\usepackage{amsmath}
\usepackage{amssymb}
\usepackage{nicefrac}
\usepackage{microtype}
\usepackage{graphicx}
\usepackage{subcaption}
\usepackage{xcolor}
\usepackage{algorithm}
\usepackage{algorithmic}
\usepackage{multirow}
\usepackage{makecell}
\usepackage{stmaryrd}
\usepackage{amsthm}
\usepackage{enumitem}
\usepackage{float}
\usepackage{tikz}
\usetikzlibrary{arrows.meta,positioning,shapes.geometric,fit}
\usepackage[numbers,sort&compress]{natbib}

\newtheorem{theorem}{Theorem}
\newtheorem{proposition}[theorem]{Proposition}
\newtheorem{corollary}[theorem]{Corollary}
\newtheorem{definition}[theorem]{Definition}

\newcommand{\sem}[1]{\llbracket #1 \rrbracket}

\renewcommand{\shorttitle}{The Neural Compiler}

\title{\textbf{The Neural Compiler:} Program-to-Network Translation\\for Hybrid Scientific Machine Learning}

\author{
  \textbf{Lucas Sheneman} \\
  Institute for Interdisciplinary Data Sciences \\
  University of Idaho \\
  \texttt{sheneman@uidaho.edu}
}

\date{}

\begin{document}

\maketitle

\begin{abstract}
Scientific machine learning frequently requires integrating known physics with unknown components: the equation form is known, but parameters or correction terms must be learned from data. Current approaches either discard the known structure entirely (neural networks), encode it as a soft penalty that the model may violate (physics-informed neural networks), or require hand-implementation for each new equation (manual PyTorch coding). We present \emph{The Neural Compiler}, a system that translates programs in a first-order expression language with Scheme syntax into frozen, differentiable PyTorch modules. These compiled modules compute \emph{exactly}: they produce the same output as the source program to floating-point precision, with exact gradients via autograd. The compiled component contributes zero approximation error on its safe domain, including outside the training distribution; in hybrid architectures, approximation error arises only from the learned components, while the compiled module encodes known physics exactly.

We evaluate the compiler across six experiments spanning algebraic equations (15 Feynman physics laws), ordinary differential equations (Lotka-Volterra, damped pendulum), a partial differential equation (1D heat equation), compositional generalization, and 3D vector mechanics. Against four baseline types (hand-coded PyTorch (same equations without the compiler), physics-informed neural networks (PINNs), neural ODEs, and pure MLPs) we find that: (1) compiled and hand-coded models produce \emph{numerically identical} results for single equations, confirming zero numerical discrepancy relative to hand-coded PyTorch; (2) compiled models recover physical constants to $<$1\% error with 1--4 trainable parameters where PINNs with 8,500+ parameters show 7--93\% error; (3) compiled modules compose with zero error at arbitrary depth while neural approximations accumulate errors up to $5.9 \times 10^9$ in high-depth chains; and (4) the compiler's value is \emph{systematic composability}: generating correct, differentiable modules from symbolic specifications rather than hand-coding each equation. The system supports 51 primitive operations including vector and matrix algebra, enabling PDE discretizations. We discuss how large language models can serve as a natural-language front end, translating physics descriptions into compilable programs and opening a path toward self-architecting scientific models.
\end{abstract}

\section{Introduction}
\label{sec:intro}

A central tension in scientific machine learning is the gap between what we \emph{know} and what we can \emph{encode}. Researchers frequently possess partial domain knowledge: the gravitational force follows an inverse-square law, but the drag coefficient is unknown; the heat equation governs diffusion, but the thermal diffusivity must be learned from data; the predator-prey dynamics follow Lotka-Volterra equations, but the rate constants are uncertain. In each case, the scientist holds a \emph{program}, a symbolic expression describing the known physics, and needs to integrate it with a trainable model that learns the rest.

This integration is surprisingly difficult. The dominant paradigm in scientific ML offers three unsatisfying options:

\paragraph{Option 1: Discard the program.} Train a neural network from scratch on input-output data. Neural ODEs~\cite{chen2018neural} learn the entire ODE right-hand side; standard MLPs learn arbitrary function mappings. This approach is general but wasteful: it must rediscover structure that the scientist already knows. A 3-layer MLP with 8,700 parameters learning the pendulum equation $\ddot{\theta} = -(g/L)\sin\theta - b\dot\theta$ must approximate both the $\sin$ function and the parametric dependence from data alone. It uses 4,000$\times$ more parameters than needed, fails on extrapolation, and its learned weights are uninterpretable.

\begin{sloppypar}
\paragraph{Option 2: Encode physics as a soft penalty.}
Physics-informed neural networks (PINNs)~\cite{raissi2019physics} add the governing equation as a regularization loss: the network is penalized for violating the physics but free to do so if the data loss decreases. This ``soft constraint'' approach is elegant in principle but fragile in practice. The balance between data and physics losses requires careful tuning. More fundamentally, soft constraints provide no guarantee: the model can, and often does, violate the physics, especially when extrapolating. In our experiments, a PINN trained on heat equation data recovers the thermal diffusivity~$\alpha$ with 93\% error, while a hard-constrained compiled model achieves 0.0\% error (Section~\ref{sec:pde}).
\end{sloppypar}

\paragraph{Option 3: Hand-code each equation.} Write the physics directly as a PyTorch \texttt{nn.Module} with trainable \texttt{nn.Parameter} entries. This produces exact computation with exact gradients, but it requires manual reimplementation for every new equation, is error-prone for complex expressions, and, critically, does not compose. A library of 15 hand-coded Feynman equations cannot be automatically chained, recombined, or extended without writing new code for each combination.

\paragraph{The compiler alternative.} We introduce a fourth option: \emph{compile} the known physics from a symbolic specification into a differentiable module, automatically. Given a program in a first-order expression language with Scheme syntax (e.g., \texttt{(* (- 0 (/ g\_L 1)) (sin theta))}) for pendulum gravity, the Neural Compiler produces a frozen PyTorch \texttt{nn.Module} that:

\begin{enumerate}[leftmargin=2em,topsep=0pt,itemsep=2pt]
    \item \textbf{Computes exactly.} The module produces the same output as the source program to floating-point precision for any input in the safe domain (Theorem~\ref{thm:correctness}).
    \item \textbf{Provides exact gradients.} Autograd through the compiled module yields the true derivative of the program, enabling gradient-based learning of exposed parameters (Theorem~\ref{thm:gradients}).
    \item \textbf{Exact on the safe domain.} The compiled component contributes zero approximation error at any input in its safe domain, including outside the training distribution; all approximation error comes from the learned components (Theorem~\ref{thm:extrapolation}).
    \item \textbf{Composes exactly.} Chains of compiled modules maintain zero error at arbitrary depth, while neural approximation errors can be amplified multiplicatively by downstream Lipschitz constants (Proposition~\ref{prop:composition}).
\end{enumerate}

For a \emph{single} equation, the compiler produces output identical to a hand-coded PyTorch implementation; our experiments confirm this to machine precision across all six experiment domains (Tables~\ref{tab:feynman}--\ref{tab:pde}). The compiler's value is not numerical superiority over hand-coding, but \emph{systematic composability}: it provides a uniform mechanism to translate any symbolic expression into a correct, differentiable, composable module. A library of compiled equations can be chained, recombined, and extended by changing text strings rather than writing new code. This is the foundation for programmatic physics-module generation, including by large language models (Section~\ref{sec:future}).\looseness=-1

\paragraph{Contributions.}
\begin{enumerate}[leftmargin=2em,topsep=0pt,itemsep=2pt]
    \item A compiler from first-order arithmetic expressions (with Scheme syntax) to frozen, differentiable PyTorch modules, supporting 51 primitive operations including vector and matrix algebra.
    \item Formal guarantees of compilation correctness, gradient exactness, compiled-component exactness, and composition error bounds.
    \item Hybrid architecture patterns (known structure with unknown parameters, known terms plus learned corrections, and compositional pipelines) demonstrated on six experiment domains.
    \item Systematic comparison against four baseline types (hand-coded PyTorch, PINNs, neural ODEs, pure MLPs) establishing that hard-constraint compilation provides substantially better parameter recovery and extrapolation than soft-constraint approaches, especially for transcendental, PDE, and compositional settings.
\end{enumerate}

\paragraph{Scope.} For any single equation, the compiled module produces output numerically identical to a hand-coded PyTorch implementation. If a practitioner has one equation and is willing to implement it manually, the compiler offers no accuracy benefit. The compiler's value is systematic generation, composability, and the ability to serve as a target for LLM-driven scientific model construction (Section~\ref{sec:future}).

\section{The Neural Compiler}
\label{sec:compiler}

\subsection{Source Language}

The source language $\mathcal{L}$ is a first-order, pure arithmetic expression language with Scheme syntax; programs are total on their safe domain (Definition~\ref{def:safe}), where partial operations (division, log, sqrt) are well-defined. Programs are built from constants $c \in \mathbb{R}$, input variables, let-bindings, conditionals, and 51 primitive operations organized in four categories:

\begin{itemize}[leftmargin=2em,topsep=2pt,itemsep=1pt]
    \item \textbf{Scalar arithmetic} (24 ops): $+, -, \times, \div$, pow, modulo, remainder, abs, min, max, sin, cos, exp, sqrt, log, comparisons ($=$, $<$, $>$, $\leq$, $\geq$), logic (and, or, not), conditional (if).
    \item \textbf{Vector operations} (9 ops): vec (construction), ref (indexing), dot, cross, norm, normalize, vsum, vlen, scale.
    \item \textbf{Matrix operations} (11 ops): mat (construction), matmul, matvec, transpose, trace, det, inv, outer, eye, zeros, ones.
    \item \textbf{Control flow} (7 ops): let-bindings, tail-recursive loops (loop/recur), and general recursion (letrec/call). Tail recursion is optimized to iterative loops during compilation; general recursion is evaluated via stack-based dispatch with a configurable depth limit.
\end{itemize}

$\mathcal{L}$ does not include higher-order functions, mutation, or I/O. Bracket syntax \texttt{[1 2 3]} desugars to \texttt{(vec 1 2 3)} for concise vector literals.

\subsection{Compilation Pipeline}
\label{sec:pipeline}

The compiler transforms source code through four stages. Figure~\ref{fig:overview} summarizes the pipeline and its use inside a hybrid scientific machine learning architecture.

\begin{figure}[!htbp]
\centering
\begin{tikzpicture}[
    box/.style={rectangle, draw, rounded corners, minimum height=0.9cm, minimum width=2.0cm, align=center, font=\small},
    arrow/.style={-{Stealth[length=2.5mm]}, thick},
    lbl/.style={font=\scriptsize, midway, above, yshift=2pt},
    scale=0.92, transform shape
]
    \node[font=\footnotesize\bfseries] at (6.5,0) (piptitle) {Compilation Pipeline (one-time, $<$150\,$\mu$s)};
    \node[box, fill=blue!10] at (0,-1) (source) {Scheme\\Source};
    \node[box, fill=blue!10] at (3.2,-1) (ast) {AST};
    \node[box, fill=blue!10] at (6.2,-1) (anf) {ANF};
    \node[box, fill=blue!10] at (9.2,-1) (graph) {Compute\\Graph};
    \node[box, fill=green!15] at (12.5,-1) (dm) {\texttt{DirectModule}};

    \draw[arrow] (source) -- (ast) node[lbl] {parse};
    \draw[arrow] (ast) -- (anf) node[lbl] {flatten};
    \draw[arrow] (anf) -- (graph) node[lbl] {build};
    \draw[arrow] (graph) -- (dm) node[lbl] {compile};


    \node[font=\small] at (0.5,-3.8) (input) {Input $x$};

    \coordinate (fork) at (2.2,-3.8);

    \node[box, fill=green!15] at (5.0,-3.1) (frozen) {Frozen\\Module};
    \node[box, fill=orange!15] at (5.0,-4.5) (mlp) {Trainable\\MLP};

    \draw[arrow] (input.east) -- (fork);
    \draw[arrow] (fork) |- (frozen.west);
    \draw[arrow] (fork) |- (mlp.west);

    \node[circle, draw, inner sep=2pt, font=\small\bfseries] at (8.8,-3.8) (plus) {$+$};

    \draw[arrow] (frozen.east) -| (plus.north)
        node[pos=0.25, above, font=\scriptsize] {known physics};
    \draw[arrow] (mlp.east) -| (plus.south)
        node[pos=0.25, below, font=\scriptsize] {learned residual};

    \node[box, fill=yellow!15] at (11.5,-3.8) (output) {Hybrid\\Output};
    \draw[arrow] (plus) -- (output);

    \draw[arrow, dashed, gray] (dm.south) -- ++(0,-0.7) -| (frozen.north)
        node[pos=0.25, above, font=\scriptsize, text=black] {freeze};

    \node[font=\footnotesize\bfseries] at (6.5,-5.5) {Hybrid Architecture (training)};
\end{tikzpicture}
\caption{System overview. \textbf{Top}: The compilation pipeline transforms Scheme source through AST and ANF representations into a frozen \texttt{DirectModule} in under 150\,$\mu$s. \textbf{Bottom}: During training, the frozen module computes the known physics component exactly, while a trainable MLP learns a residual correction. Gradients flow through both paths to train the residual model while the compiled physics module remains fixed (Corollary~\ref{cor:gradient_flow}).}
\label{fig:overview}
\end{figure}
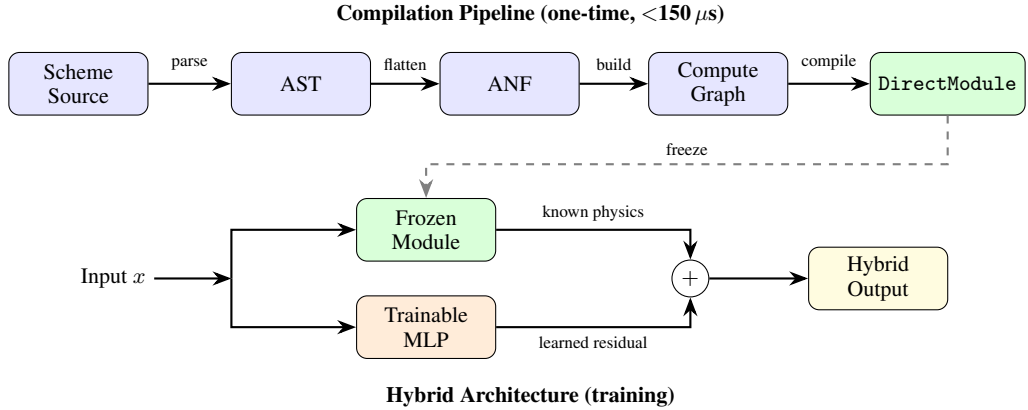

\paragraph{1. Parse $\to$ AST.} Recursive-descent parsing of S-expressions into an abstract syntax tree. Bracket syntax \texttt{[...]} is desugared to \texttt{(vec ...)} during tokenization.

\paragraph{2. ANF Transform.} The AST is converted to A-Normal Form~\cite{flanagan1993essence}, where all operation arguments are trivial (constants or variables). Compound subexpressions are let-bound to fresh temporaries. This flattening creates a one-to-one correspondence between let-bindings and computational nodes.

\paragraph{3. Tail-Call Optimization.} Tail-recursive functions are transformed to iterative loops. This is standard TCO but critical for scientific programs: many iterative algorithms (time-stepping, Newton's method, fixed-point iteration) are naturally expressed as tail recursion in Scheme but must execute iteratively for efficiency.

\paragraph{4. Graph Construction $\to$ DirectModule.} Each ANF let-binding becomes a node in a directed acyclic graph (\texttt{ComputeGraph}). Edges flow from producers to consumers. The graph is compiled into a \texttt{DirectModule}, a PyTorch \texttt{nn.Module} that evaluates nodes in topological order using instruction dispatch. Constants become fixed values; input variables are set at evaluation time; primitive operations dispatch to native PyTorch implementations. For batched evaluation, the same instruction sequence processes all batch elements in parallel via tensor broadcasting. Figure~\ref{fig:pipeline} gives a concrete example of the parse, flatten, graph-construction, and instruction-sequence lowering process.

\begin{figure}[!htbp]
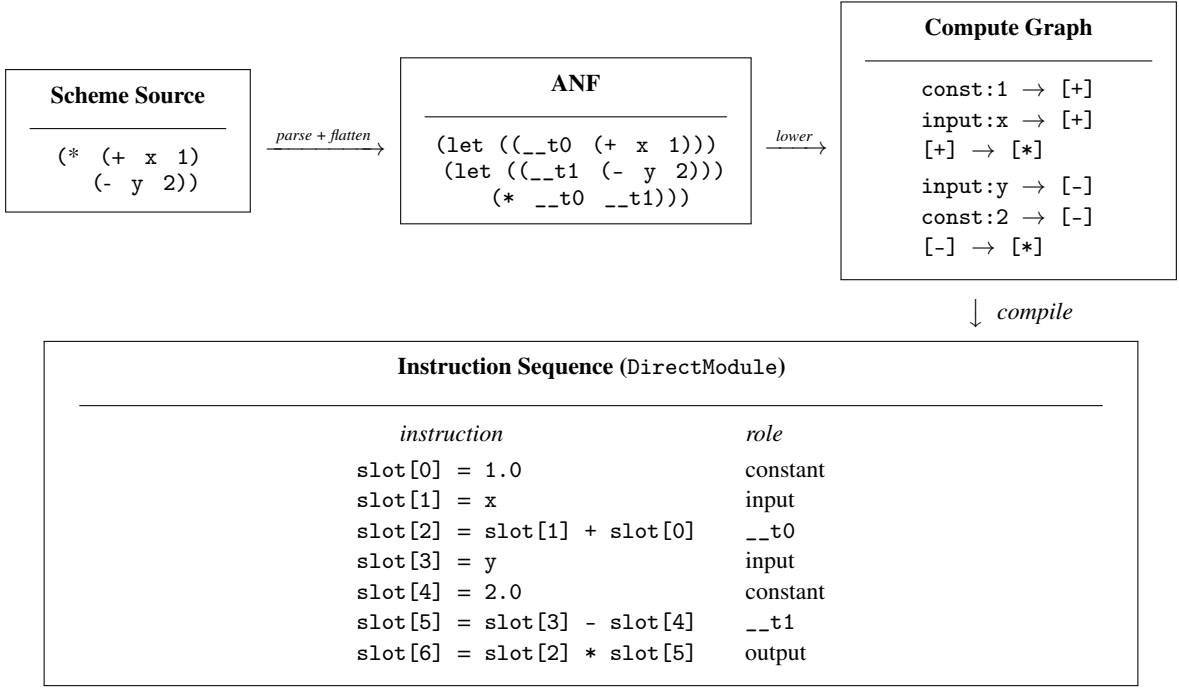

\centering
\small
\setlength{\fboxsep}{6pt}
\renewcommand{\arraystretch}{1.15}

\begin{tabular}{@{}c@{}c@{}c@{}c@{}c@{}}
%
\fbox{\parbox{2.8cm}{\centering\textbf{Scheme Source}\\\rule{2.6cm}{0.4pt}\\[3pt]
\ttfamily\footnotesize
(\textrm{*}\; (\textrm{+}\; x\; 1)\\
\;\;\; (\textrm{-}\; y\; 2))}}
&
\;\parbox{1.8cm}{\centering\footnotesize $\xrightarrow{\;\textit{parse + flatten}\;}$}\;
&
\fbox{\parbox{4.2cm}{\centering\textbf{ANF}\\\rule{4.0cm}{0.4pt}\\[3pt]
\ttfamily\footnotesize
(let ((\textunderscore\textunderscore t0\; (+\; x\; 1)))\\
\;\;(let ((\textunderscore\textunderscore t1\; (-\; y\; 2)))\\
\;\;\;\;(*\; \textunderscore\textunderscore t0\; \textunderscore\textunderscore t1)))}}
&
\;\parbox{1.0cm}{\centering\footnotesize $\xrightarrow{\;\textit{lower}\;}$}\;
&
\fbox{\parbox{4.0cm}{\centering\textbf{Compute Graph}\\\rule{3.8cm}{0.4pt}\\[3pt]
\ttfamily\footnotesize
\begin{tabular}[t]{@{}l@{}}
const:1 $\rightarrow$ {[}+{]}\\
input:x $\rightarrow$ {[}+{]}\\
{[}+{]} $\rightarrow$ {[}*{]}\\[2pt]
input:y $\rightarrow$ {[}-{]}\\
const:2 $\rightarrow$ {[}-{]}\\
{[}-{]} $\rightarrow$ {[}*{]}\\
\end{tabular}}}
\end{tabular}

\vspace{6pt}
\hfill$\big\downarrow$\;\;\textit{compile}\hspace{1.8cm}\mbox{}
\vspace{4pt}

\fbox{\parbox{0.85\columnwidth}{\centering\textbf{Instruction Sequence (\texttt{DirectModule})}\\[2pt]\rule{0.82\columnwidth}{0.4pt}\\[4pt]
\footnotesize
\begin{tabular}{@{}r@{\;\;=\;\;}l@{\qquad}l@{}}
\multicolumn{2}{@{}l}{\textit{\quad\;\; instruction}} & \textit{role} \\[2pt]
\texttt{slot[0]} & \texttt{1.0} & constant \\
\texttt{slot[1]} & \texttt{x} & input \\
\texttt{slot[2]} & \texttt{slot[1] + slot[0]} & \texttt{\_\_t0} \\
\texttt{slot[3]} & \texttt{y} & input \\
\texttt{slot[4]} & \texttt{2.0} & constant \\
\texttt{slot[5]} & \texttt{slot[3] - slot[4]} & \texttt{\_\_t1} \\
\texttt{slot[6]} & \texttt{slot[2] * slot[5]} & output \\
\end{tabular}
}}

\caption{Compilation pipeline. A Scheme expression is parsed, flattened into ANF, lowered to a compute graph, and compiled into a topologically ordered instruction sequence. The resulting \texttt{DirectModule} evaluates the sequence for both scalar and batched tensor inputs.}
\label{fig:pipeline}
\end{figure}

\subsection{Theoretical Guarantees}
\label{sec:theory}

\begin{definition}[Safe Domain]
\label{def:safe}
The safe domain $\mathcal{D}_P$ of a program $P$ is the set of inputs $x$ such that no intermediate computation encounters division by zero, square root of a negative value, or logarithm of a non-positive value. On $\mathcal{D}_P$, all primitive implementations are defined and match the corresponding PyTorch floating-point operations.
\end{definition}

\begin{theorem}[Compilation Correctness]
\label{thm:correctness}
For any program $P \in \mathcal{L}$ and input $x \in \mathcal{D}_P$, the compiled DirectModule $M$ satisfies $M(x) = \sem{P}(x)$ under the floating-point semantics of the corresponding PyTorch primitives.
\end{theorem}

\begin{proof}[Proof sketch]
By structural induction on the ANF representation. The ANF transform preserves semantics by introducing let-bindings only. For constants and variables (base cases), the instruction sequence sets slot values directly. For primitive applications $f(a_1, \ldots, a_k)$, ANF guarantees arguments are trivial, so their values are already computed in earlier slots. The instruction dispatches to the PyTorch implementation $\hat{f}$, which equals $f$ on $\mathcal{D}_P$. Topological ordering ensures all operands are available before each operation. For loops, the iterative evaluation of the body applies the theorem inductively at each iteration.
\end{proof}

\begin{theorem}[Gradient Correctness]
\label{thm:gradients}
For any differentiable program $P$ and input $x$ in the interior of $\mathcal{D}_P$, away from nondifferentiable branch boundaries (conditionals, \texttt{abs}, \texttt{min}, \texttt{max}) and singular points of matrix operations (\texttt{det}, \texttt{inv}): $\frac{\partial M}{\partial x_i}(x) = \frac{\partial \sem{P}}{\partial x_i}(x)$.
\end{theorem}

\begin{proof}[Proof sketch]
The compiled module composes standard differentiable PyTorch operations. By Theorem~\ref{thm:correctness}, the forward computation is identical to direct evaluation. PyTorch's autograd applies the chain rule through the same operation sequence, yielding identical derivatives.
\end{proof}

\begin{corollary}[Gradient Flow Through Frozen Subgraphs]
\label{cor:gradient_flow}
Let $H(x; \theta) = M(f_\theta(x))$ be a hybrid model with compiled module $M$ and trainable network $f_\theta$. Then $\frac{\partial H}{\partial \theta} = \frac{\partial M}{\partial z}\big|_{z=f_\theta(x)} \cdot \frac{\partial f_\theta}{\partial \theta}$, with both factors computed exactly by autograd. The frozen module provides exact gradients without any learned parameters.
\end{corollary}

\begin{theorem}[Compiled Component Exactness]
\label{thm:extrapolation}
Let $H(x; \theta) = M(\pi(x; \theta))$ be a hybrid model with compiled module $M$ and learned projection $\pi$. If training converges to $\theta^*$ such that $\pi(x; \theta^*) = \pi^*(x)$ is the true mapping, then $H(x; \theta^*) = \sem{P}(\pi^*(x))$ for all $x \in \mathcal{D}_{P \circ \pi^*}$, including inputs outside the training distribution.
\end{theorem}

\begin{proof}[Proof sketch]
By Theorem~\ref{thm:correctness}, $M(z) = \sem{P}(z)$ for all $z \in \mathcal{D}_P$. Substituting $z = \pi^*(x)$ gives the result. The compiled component contributes zero approximation error; all error comes from the learned projection.
\end{proof}

\begin{proposition}[Composition Error Bound]
\label{prop:composition}
Let $M_1, \ldots, M_k$ be compiled (exact) modules and $\hat{M}_1, \ldots, \hat{M}_k$ be neural approximations with per-module error $\epsilon_i = \sup_z |M_i(z) - \hat{M}_i(z)|$ and downstream Lipschitz constants $L_j$. The composition error satisfies:
$$|M_k \circ \cdots \circ M_1(x) - \hat{M}_k \circ \cdots \circ \hat{M}_1(x)| \leq \sum_{i=1}^{k} \epsilon_i \prod_{j=i+1}^{k} L_j$$
For compiled modules, $\epsilon_i = 0$ for all $i$, yielding zero composition error. For neural approximations of polynomial modules (degree $d$), the error grows as $O(\epsilon \cdot L^{k-1})$, growing exponentially in chain depth.
\end{proposition}

\begin{proposition}[Parameter Efficiency]
\label{prop:param_efficiency}
A compiled hybrid model with $m$ symbolic constants has exactly $m$ trainable parameters. An MLP with hidden width $h$ and $L$ layers requires $O(nh + Lh^2)$ parameters for $n$-dimensional input. For the Feynman benchmark: $m \in \{1, 2, 3\}$ vs.\ $O(12{,}700)$ MLP parameters, a ratio of $4{,}200\times$ to $12{,}700\times$.
\end{proposition}

\subsection{Hybrid Architecture Patterns}
\label{sec:patterns}

The compiled module integrates into trainable architectures via three patterns:

\paragraph{Pattern 1: Known structure, unknown parameters.} The program encodes the full equation with symbolic constants exposed as \texttt{nn.Parameter}. Training recovers constants via gradient descent through the frozen module. Example: $F = G \cdot m_1 m_2 / r^2$ with trainable $G$.

\paragraph{Pattern 2: Known term + learned correction.} The compiled module handles one term; a trainable MLP handles the remainder. Outputs are summed. Example: $\dot\theta = \underbrace{-(g/L)\sin\theta}_\text{compiled} + \underbrace{f_\theta(\theta, \dot\theta)}_\text{MLP}$.

\paragraph{Pattern 3: Compositional pipeline.} Multiple programs are compiled into a library of modules that can be chained. Example: compiling $\text{square} \to \text{add\_one} \to \text{cube}$ as three modules with zero composition error.

\section{Experiments}
\label{sec:experiments}

We evaluate the Neural Compiler across six experiment domains, comparing against four baseline types:

\begin{itemize}[leftmargin=2em,topsep=2pt,itemsep=1pt]
    \item \textbf{Compiled}: program compiled from Scheme via the Neural Compiler (ours)
    \item \textbf{Hand-coded PyTorch}: identical physics equation written directly as \texttt{nn.Module} (no compiler)
    \item \textbf{PINN}: MLP approximates the solution; physics equation as soft loss via autograd
    \item \textbf{MLP}: standard neural network learns the mapping from data alone
    \item \textbf{Neural ODE}: MLP parameterizes the ODE right-hand side (where applicable)
\end{itemize}

The hand-coded baseline is critical: it isolates the compiler's contribution from the benefit of hard physics constraints. Any accuracy difference between compiled and hand-coded models would indicate compiler overhead; our experiments show the difference is zero to machine precision.

\paragraph{PINN implementation details.} All PINNs use a 3-layer MLP with 64 hidden units and tanh activations, trained with Adam (lr $= 10^{-3}$). The physics loss adds the PDE or ODE residual computed via autograd, weighted equally with the data loss ($\lambda_{\text{data}} = \lambda_{\text{physics}} = 1.0$). No adaptive loss weighting~\cite{wang2021understanding}, input normalization, or architecture search was applied. We acknowledge that PINNs are sensitive to these choices; our comparison reflects a standard configuration rather than an optimally tuned baseline. These results should be interpreted as evidence that standard PINN formulations can be fragile under basic settings, not as a claim that no tuned PINN variant could improve the reported numbers.

\subsection{Experiment 1: Feynman Equation Coefficient Learning}
\label{sec:feynman}

\begin{sloppypar}
We evaluate on 15 equations from the Feynman Symbolic Regression benchmark~\cite{udrescu2020ai}, spanning mechanics, thermodynamics, electromagnetism, relativity, and quantum mechanics. Unlike symbolic regression, we assume the equation form is known and compile it; the task is to recover physical constants from noisy data.
\end{sloppypar}

\paragraph{Setup.} Each equation is compiled from Scheme source with 1--3 symbolic constants as trainable parameters. The hand-coded baseline implements the identical equation as a PyTorch function with the same \texttt{nn.Parameter} entries. The MLP baseline uses a 3-layer network (64 hidden units, ReLU, 12,673--12,865 parameters). All models train for 3,000 epochs on 10,000 samples per epoch with 2\% Gaussian noise.

\begin{table}[!htbp]
\caption{Feynman equation coefficient learning. Compiled and hand-coded models (1--3 params) produce identical MSE (``C/HC'' columns). Both achieve orders-of-magnitude improvement over the MLP baseline ($\sim$12,700 params) on extrapolation.}
\label{tab:feynman}
\centering
\footnotesize
\begin{tabular}{@{}lrrrrrr@{}}
\toprule
& & \multicolumn{2}{c}{MSE (in-dist)} & \multicolumn{2}{c}{MSE ($5\times$ extrap)} & Coeff. \\
\cmidrule(lr){3-4} \cmidrule(lr){5-6}
Equation & \#P & C/HC & MLP & C/HC & MLP & err. \\
\midrule
Planck $E = hf$ & 1 & 2.3e-8 & 1.6e-3 & 2.4e-6 & 3.0e+3 & 0.003\% \\
Hooke $F = -kx$ & 1 & 2.2e-7 & 9.7e-4 & 3.5e-5 & 8.9e+2 & 0.013\% \\
Kinetic $\frac{1}{2}mv^2$ & 1 & 6.1e-7 & 5.2e-3 & 4.3e-4 & 6.6e+3 & 0.072\% \\
Gravity $Gm_1m_2/r^2$ & 1 & 8.7e-8 & 3.6e-2 & 1.2e-5 & 2.2e+4 & 0.010\% \\
Ideal gas $nRT$ & 1 & 5.3e-8 & 1.4e-3 & 1.5e-5 & 3.5e+3 & 0.002\% \\
Pendulum $k\sqrt{L/g}$ & 1 & 3.9e-8 & 1.6e-3 & 3.1e-7 & 2.9e+2 & 0.000\% \\
Heat $mc\Delta T$ & 1 & 5.3e-7 & 1.2e-3 & 3.3e-5 & 9.3e+2 & 0.008\% \\
Coulomb $k_eq_1q_2/r^2$ & 1 & 4.2e-5 & 6.2e-5 & 1.0e-2 & 1.8e-1 & 0.804\% \\
\midrule
Gaussian $\mathcal{N}(\mu,\sigma)$ & 2 & 2.9e-10 & 3.9e-4 & 1.3e-7 & 8.7e+2 & 0.003\% \\
Rel.\ energy & 2 & 5.1e-7 & 1.1e-3 & 7.3e-5 & 5.1e+2 & 0.003\% \\
Sound $\sqrt{\gamma P/\rho}$ & 2 & 9.1e-9 & 1.1e-3 & 6.2e-8 & 2.3e+2 & 0.010\% \\
Barometric & 3 & 1.2e-8 & 4.9e-3 & 4.2e-7 & 1.7e+4 & 0.058\% \\
E-field $\epsilon E^2V/2$ & 2 & 1.7e-7 & 1.2e-3 & 9.1e-5 & 4.7e+3 & 0.016\% \\
\midrule
Oscillator & 3 & 5.0e-1 & 3.9e-1 & 5.8e-1 & 4.3e-1 & $>$95\% \\
Lorentz & 1 & 7.5e-3 & 8.5e-3 & 1.6e-2 & 1.6e-2 & 19.8\% \\
\bottomrule
\end{tabular}
\end{table}

\paragraph{Results.} Table~\ref{tab:feynman} shows that compiled and hand-coded models produce numerically identical MSE across all 15 equations, confirming zero numerical discrepancy relative to hand-coded PyTorch. Of the 15 equations, 13 recover constants to $<$1\% error with 1--3 parameters, achieving a median $4{,}463\times$ improvement over the MLP baseline in-distribution and $143{,}000{,}000\times$ on $5\times$ extrapolation. Two equations fail: the harmonic oscillator (periodic loss landscape traps frequency optimization) and the Lorentz factor ($v \to c$ singularity creates flat gradients).

\paragraph{Sample efficiency.} The compiled model achieves near-zero MSE from just 10 training samples; the known equation structure means a single constant can be pinned from minimal data. The MLP requires $>$10,000 samples and still fails on extrapolation. Figure~\ref{fig:feynman} visualizes these results across error, extrapolation, coefficient recovery, parameter count, and improvement ratio.

\begin{figure}[!htbp]
\centering
\includegraphics[width=0.92\textwidth]{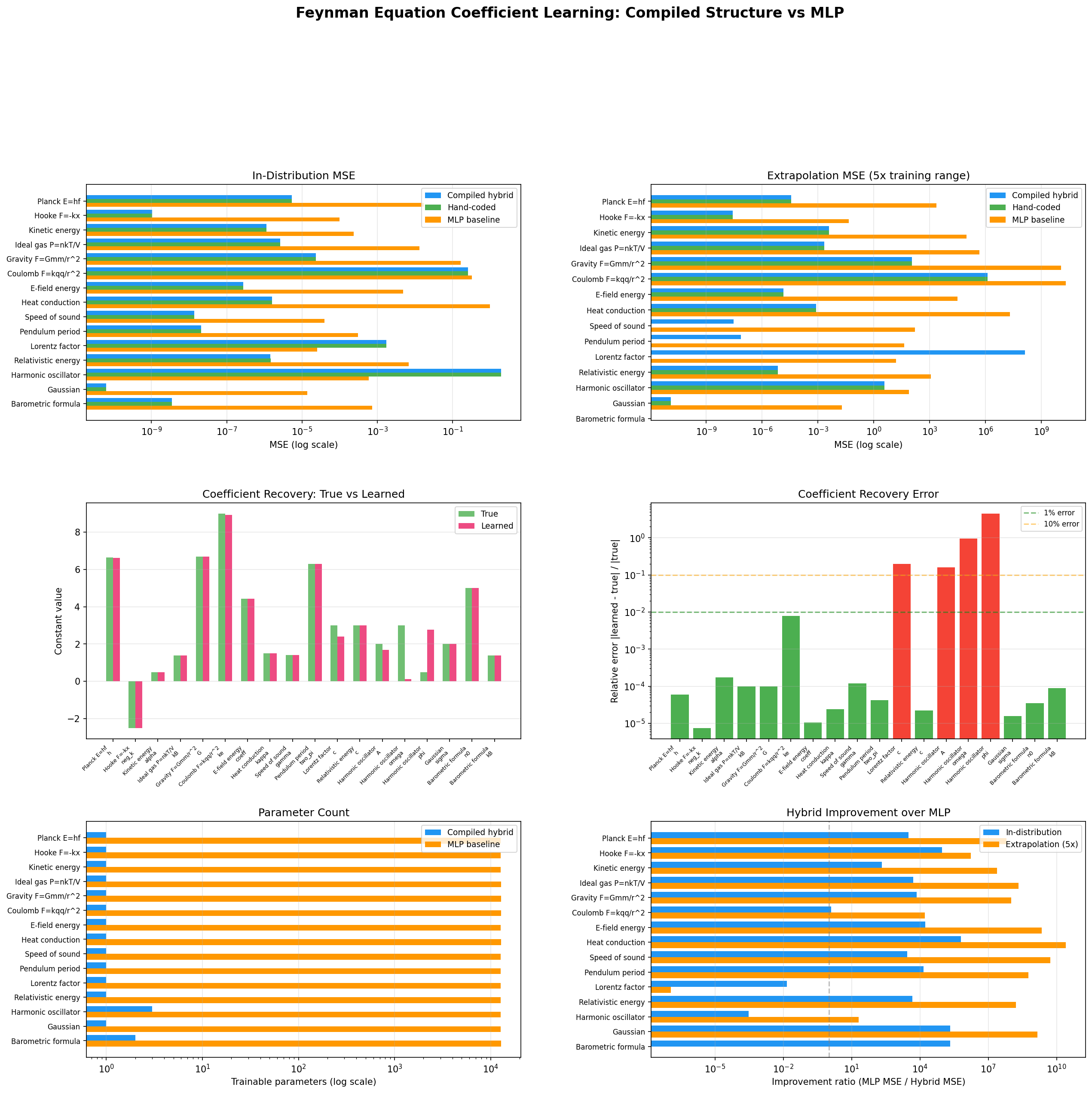}
\caption{Feynman equation coefficient learning results. \textbf{Top left}: In-distribution MSE across benchmark equations; compiled and hand-coded hybrid models usually closely match and generally outperform the MLP baseline. \textbf{Top right}: Extrapolation MSE evaluated over a $5\times$ training range, where hybrid models usually show stronger generalization than the MLP baseline. \textbf{Middle left}: True vs.\ learned coefficient values. \textbf{Middle right}: Relative coefficient recovery error on a log scale; most coefficients are recovered accurately, but several small-magnitude parameters show large relative recovery errors. \textbf{Bottom left}: Parameter count comparison on a log scale, showing that compiled hybrids use far fewer trainable parameters than the MLP baseline. \textbf{Bottom right}: Hybrid improvement ratio over the MLP for in-distribution and extrapolation settings; most equations show substantial improvement, though not uniformly across all cases.}
\label{fig:feynman}
\end{figure}

\subsection{Experiment 2: Lotka-Volterra Predator-Prey ODE}
\label{sec:lv}

The system $\dot{x} = \alpha x - \beta xy$, $\dot{y} = \delta xy - \gamma y$ models predator-prey dynamics with 4 rate constants. We train via RK4 integration with multiple shooting~\cite{bock1984multiple}.

\paragraph{Setup.} \emph{Compiled}: both RHS equations compiled from Scheme with $\alpha, \beta, \delta, \gamma$ as trainable parameters. \emph{Hand-coded}: identical equations as PyTorch functions. \emph{PINN}: MLP approximates $x(t), y(t)$ with ODE residual as soft loss. \emph{MLP}: 3-layer network (64 hidden units) learns the full RHS. All models trained for 3,000 epochs on 2\% noisy trajectory observations over $t \in [0, 12]$.

\begin{table}[!htbp]
\caption{Lotka-Volterra ODE results. Compiled and hand-coded produce identical parameter recovery ($<$1.1\% error, 4 params). The PINN (8,582 params) shows 1--8\% error and catastrophic extrapolation failure.}
\label{tab:lv}
\centering
\footnotesize
\begin{tabular}{@{}lrrrrrrr@{}}
\toprule
& & \multicolumn{3}{c}{Param.\ Recovery (\% err)} & \multicolumn{3}{c}{Trajectory MSE} \\
\cmidrule(lr){3-5} \cmidrule(lr){6-8}
Model & Params & $\alpha$ & $\beta$ & $\delta$/$\gamma$ & In-dist & $2\times$ & $5\times$ \\
\midrule
Compiled & 4 & 0.67 & 0.59 & 1.08/0.23 & 2.4e-3 & 4.9e-3 & 2.2e-2 \\
Hand-coded & 4 & 0.67 & 0.59 & 1.08/0.23 & 2.4e-3 & 4.9e-3 & 2.2e-2 \\
PINN & 8{,}582 & 7.79 & 7.21 & 2.22/1.48 & 7.8e-4 & 2.6e+0 & 4.9e+0 \\
MLP & 8{,}642 & -- & -- & -- & 4.9e-3 & 9.4e-3 & 4.2e-2 \\
Neural ODE & 8{,}642 & -- & -- & -- & 4.9e-3 & 9.5e-3 & 4.3e-2 \\
\bottomrule
\end{tabular}
\end{table}

\paragraph{Results (Table~\ref{tab:lv}).} Compiled and hand-coded models recover all four rate constants to $<$1.1\% error with 4 trainable parameters. The PINN, despite having $2{,}146\times$ more parameters, shows 1.5--7.8\% recovery error and catastrophic extrapolation failure ($5\times$ MSE $= 4.9$ vs.\ $0.022$). The MLP and Neural ODE perform similarly to the compiled model in-distribution (Lotka-Volterra is polynomial, well within MLP approximation capacity) but neither provides interpretable parameter recovery.

\paragraph{Finding.} The compiled advantage for polynomial ODEs is \emph{interpretability}, not accuracy: 4 physically meaningful parameters vs.\ 8,642 opaque weights. The accuracy advantage emerges for transcendental dynamics (Section~\ref{sec:pendulum}). Figure~\ref{fig:lv} shows the trajectory fits, extrapolation behavior, loss curves, parameter recovery, and ranked long-horizon errors; Figure~\ref{fig:lv_phase} shows the corresponding phase-space trajectories under $5\times$ extrapolation.

\begin{figure}[!htbp]
\centering
\includegraphics[width=\textwidth]{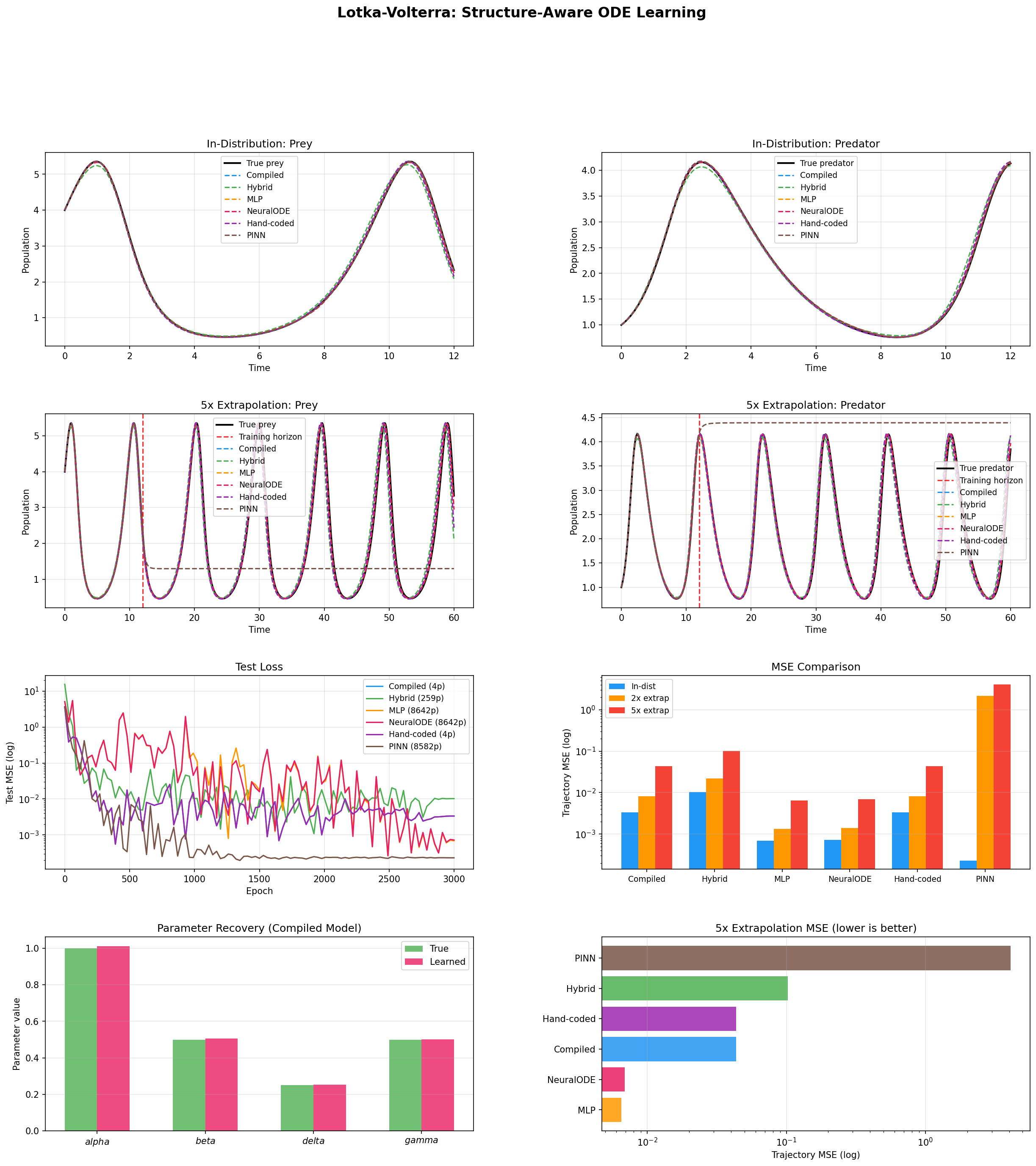}
\caption{Lotka--Volterra predator-prey ODE. \textbf{Row~1}: In-distribution trajectory fit for prey (left) and predator (right) with 2\% observation noise. \textbf{Row~2}: $5\times$ extrapolation beyond the training horizon (red dashed line); the PINN fails to preserve oscillatory dynamics and collapses to an incorrect nearly constant trajectory. \textbf{Row~3, left}: Test loss curves with parameter counts in parentheses. \textbf{Row~3, right}: Trajectory MSE across in-distribution, $2\times$, and $5\times$ extrapolation; purely learned baselines can achieve competitive or lower trajectory MSE despite lacking interpretable parameters. \textbf{Row~4, left}: Parameter recovery for the compiled known-structure model, showing true vs.\ learned values of $\alpha$, $\beta$, $\delta$, $\gamma$. \textbf{Row~4, right}: Ranked $5\times$ extrapolation MSE across all models. The PINN achieves low test loss but poor long-horizon extrapolation, suggesting that it fits the training-window objective without learning the underlying ODE structure.}
\label{fig:lv}
\end{figure}

\begin{figure}[!htbp]
\centering
\includegraphics[width=\textwidth]{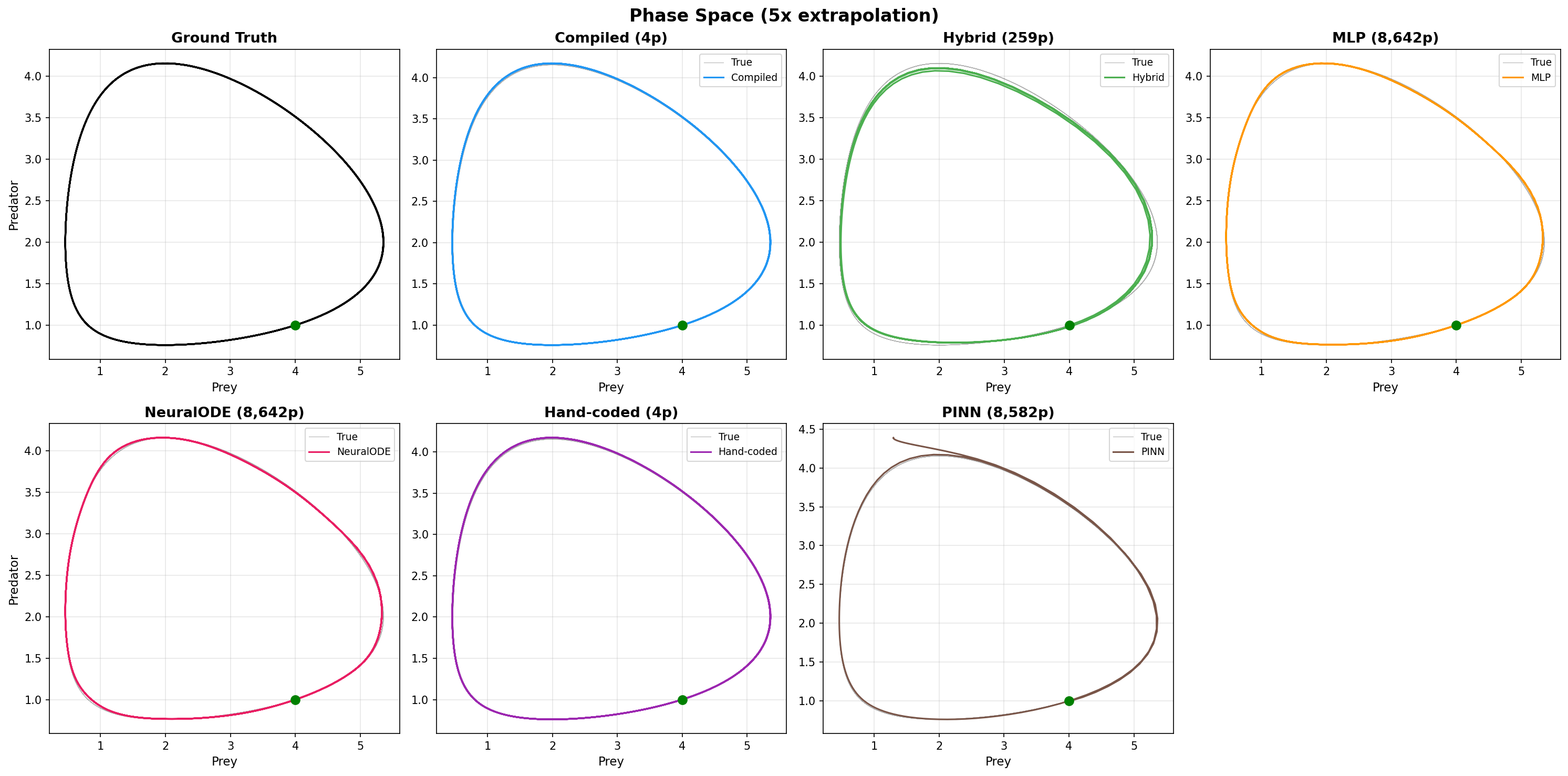}
\caption{Lotka-Volterra phase-space trajectories under $5\times$ extrapolation. Ground truth, compiled known-structure model (4 parameters), hybrid model (259 parameters), MLP (8,642 parameters), Neural ODE (8,642 parameters), hand-coded model (4 parameters), and PINN (8,582 parameters). The compiled known-structure, hybrid, and hand-coded models preserve the qualitative limit-cycle structure, while the purely learned baselines show increasing distortion or phase-space mismatch under extrapolation.}
\label{fig:lv_phase}
\end{figure}

\subsection{Experiment 3: Damped Pendulum ODE}
\label{sec:pendulum}

The system $\ddot{\theta} = -(g/L)\sin\theta - b\dot{\theta} + F(t)$ involves a transcendental ($\sin$) that MLPs approximate poorly.

\paragraph{Setup.} Same training protocol as Lotka-Volterra. The unforced variant ($F = 0$) enables all five baselines.

\begin{table}[!htbp]
\caption{Damped pendulum ODE results (unforced). The compiled model with 2 parameters achieves $731\times$ better in-distribution MSE than the 8,706-parameter MLP. The transcendental $\sin$ operation creates an accuracy gap that additional MLP capacity cannot close. The PINN recovers $g/L$ with 41\% error; compiled/hand-coded achieve identical results.}
\label{tab:pendulum}
\centering
\small
\begin{tabular}{@{}lrrrrrr@{}}
\toprule
& & \multicolumn{2}{c}{Parameter Recovery} & \multicolumn{3}{c}{Trajectory MSE} \\
\cmidrule(lr){3-4} \cmidrule(lr){5-7}
Model & Params & $g/L$ err & $b$ err & In-dist & $2\times$ & $5\times$ \\
\midrule
Compiled (S1) & 2 & 0.08\% & 0.17\% & 1.5e-5 & 4.0e-6 & 5.0e-6 \\
Hand-coded & 2 & 0.08\% & 0.17\% & 1.5e-5 & 4.0e-6 & 5.0e-6 \\
PINN & 8{,}580 & 41.1\% & 16.1\% & 1.9e-1 & 1.2e-1 & 5.3e-2 \\
Compiled Hybrid (S2) & 1{,}218 & 47.4\% & -- & 1.2e-3 & 7.7e-3 & 5.0e-3 \\
MLP & 8{,}706 & -- & -- & 2.2e-2 & 2.1e-2 & 1.5e-2 \\
Neural ODE & 8{,}706 & -- & -- & 2.2e-2 & 2.1e-2 & 1.5e-2 \\
\bottomrule
\end{tabular}
\end{table}

\paragraph{Results (Table~\ref{tab:pendulum}).} The full-structure compiled model (S1) recovers $g/L$ to 0.08\% error and achieves $731\times$ better in-distribution MSE than the MLP with $4{,}353\times$ fewer parameters. The PINN recovers $g/L$ with 41\% error; soft constraints are insufficient for the nonlinear $\sin$ dynamics.

\paragraph{Transcendental advantage.} The compiled $\sin$ operation provides an inductive bias that no amount of MLP capacity can match: ReLU networks approximate $\sin$ poorly outside their training domain, while the compiled module computes it exactly everywhere.

\paragraph{Credit assignment in hybrid models.} In Scenario 2 (compiled gravity + learned damping), $g/L$ converges to 5.16 instead of 9.81 (47\% error). The MLP absorbs part of the gravity term, creating a credit assignment problem when compiled and learned components have overlapping functional forms. This is an inherent limitation of additive hybrid architectures, not specific to our compiler. Figure~\ref{fig:pendulum} illustrates this setting, including the credit-assignment failure in the learned $g/L$ parameter.

\begin{figure}[!htbp]
\centering
\includegraphics[width=\textwidth]{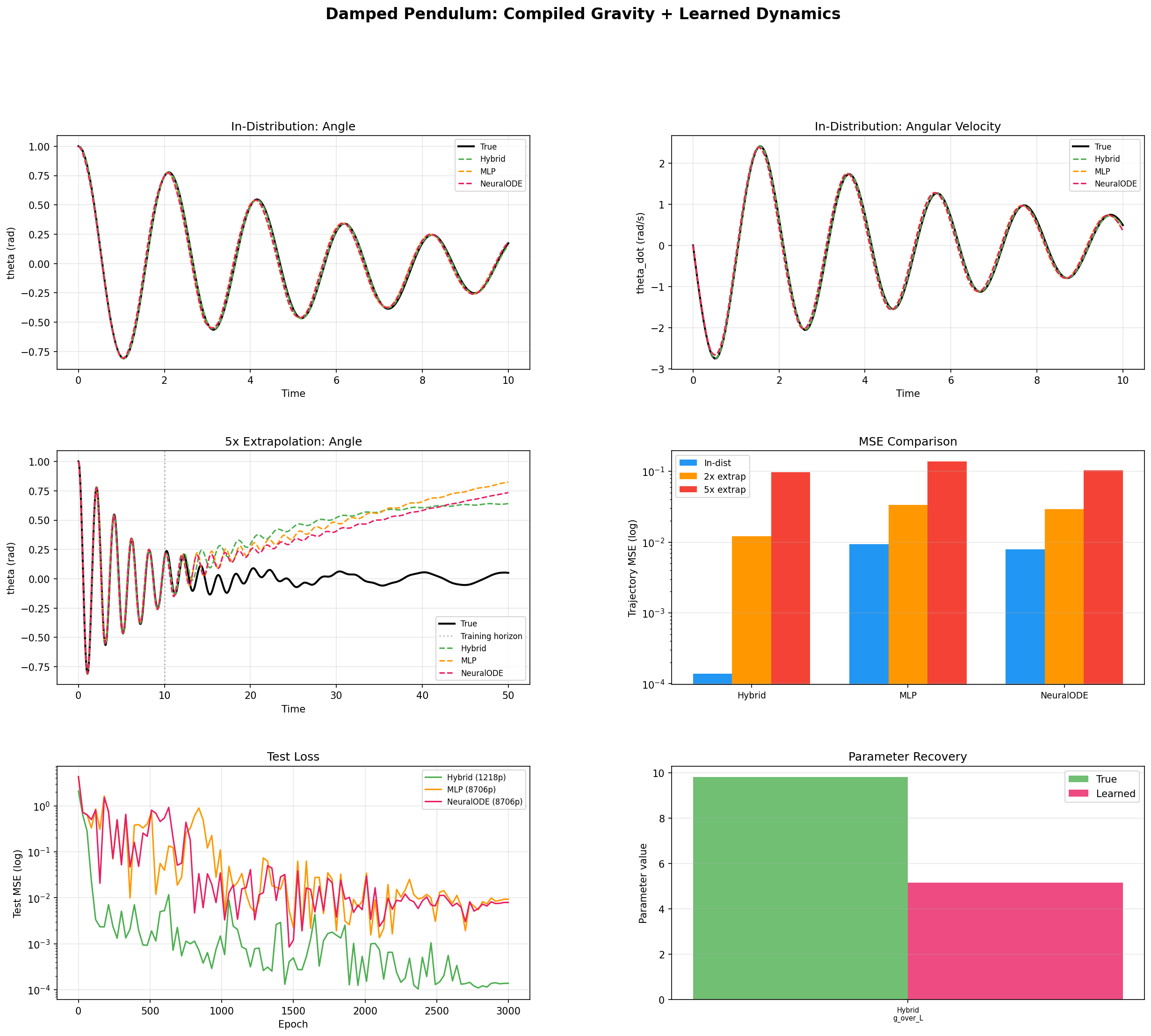}
\caption{Damped pendulum ODE, Scenario 2: compiled gravitational structure with learned dynamics. \textbf{Top left}: In-distribution angle trajectory fit for the hybrid, MLP, and Neural ODE models. \textbf{Top right}: In-distribution angular velocity trajectory fit. \textbf{Middle left}: $5\times$ extrapolation angle trajectories; the compiled hybrid extrapolates more stably than the MLP and Neural ODE baselines, though all learned models drift outside the training horizon. \textbf{Middle right}: Trajectory MSE comparison across in-distribution, $2\times$, and $5\times$ extrapolation for all three models. \textbf{Bottom left}: Test loss comparison: the hybrid model achieves the lowest test loss, reaching the $10^{-4}$ range, while the MLP and Neural ODE remain near $10^{-2}$. \textbf{Bottom right}: Parameter recovery for the hybrid model's learned $g/L$, compared with the true value; the learned value remains far from the true value, illustrating the credit-assignment issue discussed in the text.}
\label{fig:pendulum}
\end{figure}

\subsection{Experiment 4: 1D Heat Equation (PDE)}
\label{sec:pde}

The heat equation $u_t = \alpha \nabla^2 u$ is discretized on $N = 10$ grid points as $u_\text{new} = u + \Delta t \cdot \alpha \cdot L \cdot u$, where $L$ is the tridiagonal Laplacian matrix. This demonstrates the compiler on vector/matrix operations and PDE discretizations.

\paragraph{Setup.} True $\alpha = 0.01$. Training data: 50 initial conditions evolved for 5 time steps. Evaluation: 20 held-out conditions, tested at both the training horizon (interpolation) and $4\times$ the training horizon (extrapolation). Models:

\begin{itemize}[leftmargin=2em,topsep=2pt,itemsep=1pt]
    \item \emph{Compiled}: \texttt{(+ u (scale (* dt alpha) (matvec L u)))} with trainable $\alpha$ (1 param)
    \item \emph{Hand-coded}: \texttt{u + dt * alpha * torch.matmul(L, u)} with \texttt{nn.Parameter(alpha)} (1 param)
    \item \emph{PINN}: 3-layer MLP (64 hidden units, tanh activation) takes $(x, t) \to u$; PDE residual $u_t - \alpha u_{xx}$ computed via autograd and added as a physics loss with equal weighting ($\lambda_{\text{data}} = \lambda_{\text{physics}} = 1.0$); 1,000 collocation points per epoch (8,578 params). No input normalization or adaptive loss weighting was applied.
    \item \emph{MLP}: 3-layer network maps $u \to u_\text{next}$ (9,674 params)
\end{itemize}

\begin{table}[!htbp]
\caption{1D heat equation PDE results. Hard-constrained models recover $\alpha$ to machine precision (0.00\% error, 1 param); the PINN (8,578 params) shows 93\% error. Exp.\ 2 adds an unknown source term.}
\label{tab:pde}
\centering
\small
\begin{tabular*}{0.85\textwidth}{@{\extracolsep{\fill}}lrccc@{}}
\toprule
\multicolumn{5}{c}{\textbf{Experiment 1: Thermal Diffusivity Recovery}} \\
\midrule
Model & Params & $\alpha$ error & MSE (interp) & MSE (extrap) \\
\midrule
Compiled & 1 & 0.00\% & 3.7e-15 & 3.8e-15 \\
Hand-coded & 1 & 0.00\% & 3.7e-15 & 3.8e-15 \\
PINN & 8{,}578 & 92.7\% & 7.6e-1 & 5.5e-1 \\
MLP & 9{,}674 & -- & 2.4e-2 & 7.6e-2 \\
\midrule
\multicolumn{5}{c}{\textbf{Experiment 2: Hybrid Diffusion + Source Term}} \\
\midrule
Model & Params & $\alpha$ (learned) & \multicolumn{2}{c}{Final training loss} \\
\midrule
Compiled hybrid & $\sim$2{,}100 & 0.00958 & \multicolumn{2}{c}{1.2e-6} \\
Hand-coded hybrid & $\sim$2{,}100 & 0.00997 & \multicolumn{2}{c}{1.5e-6} \\
Pure MLP & 9{,}674 & -- & \multicolumn{2}{c}{1.4e-4} \\
\bottomrule
\end{tabular*}
\end{table}

\paragraph{Results (Table~\ref{tab:pde}).} The compiled and hand-coded models recover $\alpha$ to machine precision (0.0000\% error, MSE $\sim 10^{-15}$) with a single trainable parameter. The PINN, with $8{,}578\times$ more parameters, shows 92.7\% error on $\alpha$ recovery; the soft physics constraint allows the network to satisfy the data loss without correctly learning the diffusivity. This is the starkest demonstration of the hard vs.\ soft constraint gap.

\paragraph{Hybrid source experiment.} When data includes both diffusion and an unknown source term $s(x)$, the hybrid models (compiled or hand-coded diffusion + MLP source) achieve $100\times$ lower training loss than the pure MLP, confirming that encoding the known PDE structure helps even when part of the dynamics is unknown. Figure~\ref{fig:heat} summarizes the heat-equation experiments, including diffusivity recovery, test loss, extrapolation, and the diffusion-plus-source hybrid case.

\begin{figure}[!htbp]
\centering
\includegraphics[width=\textwidth]{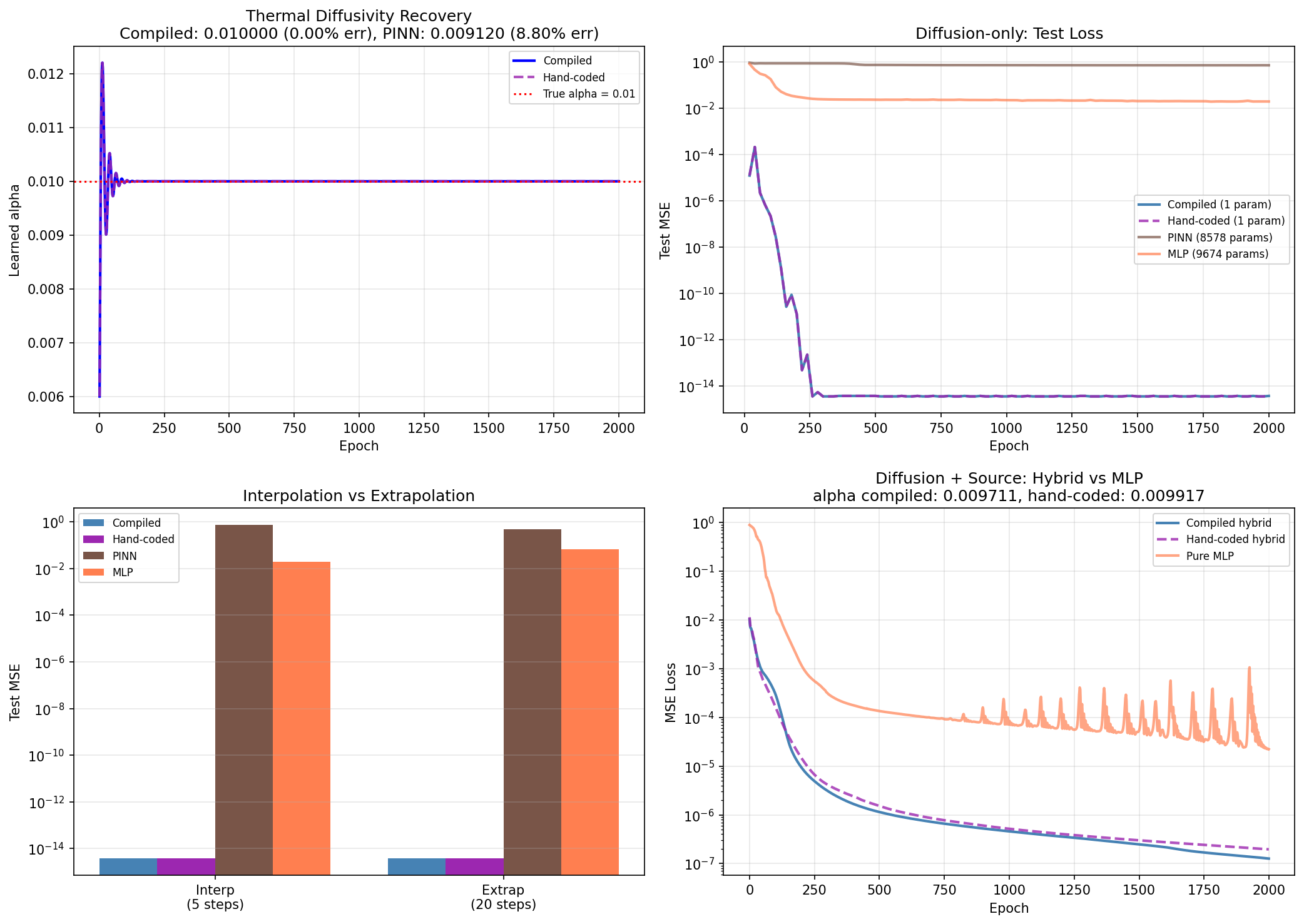}
\caption{1D heat equation PDE. \textbf{Top left}: Thermal diffusivity $\alpha$ recovery: compiled and hand-coded converge to the true value (0.01) within hundreds of epochs. \textbf{Top right}: Test loss curves for Experiment 1 (diffusion only): compiled/hand-coded (1 param) reach machine precision ($10^{-15}$); PINN (8,578 params) stalls at $10^{-1}$; MLP (9,674 params) reaches $10^{-2}$. \textbf{Bottom left}: Interpolation vs.\ extrapolation bar chart: compiled/hand-coded show no degradation under extrapolation. \textbf{Bottom right}: Experiment 2 (diffusion + source): compiled and hand-coded hybrids both outperform the pure MLP by $100\times$.}
\label{fig:heat}
\end{figure}

\subsection{Experiment 5: 3D Vector Mechanics}
\label{sec:vector}

The gravitational force $\mathbf{F} = -G m_1 m_2 / |\mathbf{r}|^3 \cdot \mathbf{r}$ is a 3D vector computation involving norm, scale, and vector arithmetic. This demonstrates the compiler's vector operation support.

\paragraph{Setup.} \emph{Compiled}: Scheme program using \texttt{norm}, \texttt{scale}, vector arithmetic with trainable $G$. \emph{Hand-coded}: same formula in PyTorch. \emph{MLP}: 3-layer network (8,899 params) maps $(m_1, m_2, \mathbf{r}) \to \mathbf{F}$. All models trained for 3,000 epochs.

\paragraph{Results.} Compiled and hand-coded produce identical results: $G$ recovered to 0.02\% error (6.675 vs.\ true 6.674), test MSE $= 1.7 \times 10^{-6}$. The MLP achieves MSE = 4.16, a $2.5 \times 10^6\times$ gap. This confirms the compiler handles vector operations correctly and that the hard physics constraint provides a strong inductive bias even in 3D. Figure~\ref{fig:vector} shows the gravitational constant recovery and gravity-plus-drag learning curves.

\begin{figure}[!htbp]
\centering
\includegraphics[width=\textwidth]{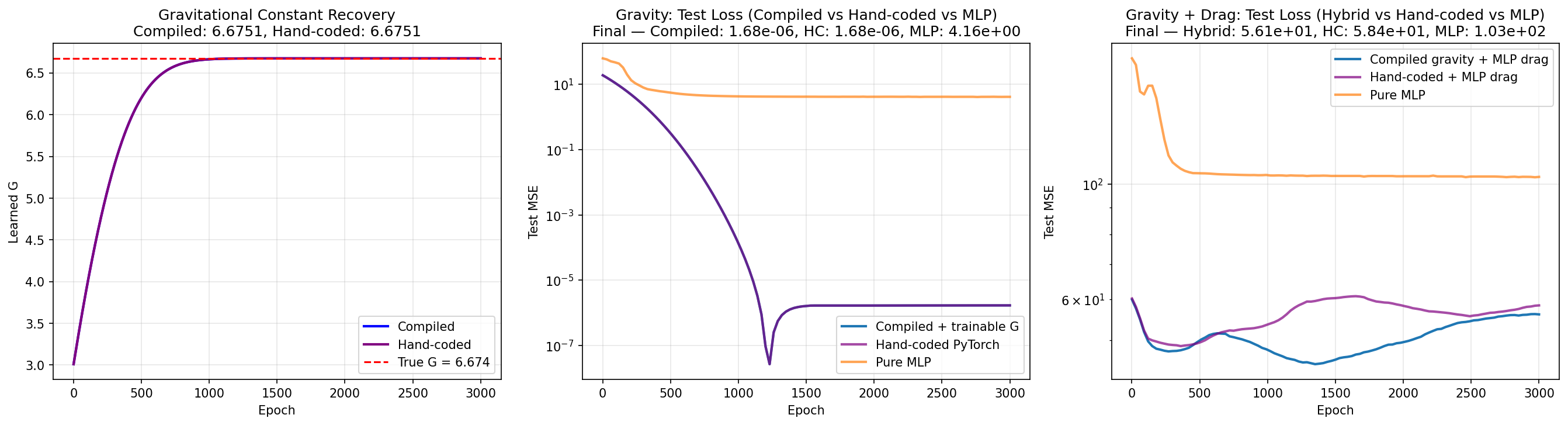}
\caption{3D vector mechanics. \textbf{Left}: Gravitational constant $G$ recovery: compiled and hand-coded converge identically to $G = 6.675$ (true $G = 6.674$, 0.02\% error). \textbf{Center}: Experiment 1 (pure gravity) test loss: compiled and hand-coded (1 param each) reach test MSE $\sim 10^{-6}$; the MLP (8,899 params) plateaus at $\sim$4, a $2.5 \times 10^6\times$ gap. \textbf{Right}: Experiment 2 (gravity + drag) test loss: compiled and hand-coded hybrids both outperform the pure MLP baseline.}
\label{fig:vector}
\end{figure}

\subsection{Experiment 6: Compositional Generalization}
\label{sec:composition}

This experiment directly tests the composition guarantee (Proposition~\ref{prop:composition}).

\paragraph{Setup.} Eight mathematical operations (square, cube, sin, exp, add\_one, negate, double, sqrt\_abs) are each compiled as frozen modules AND separately approximated by trained MLPs ($1 \to 32 \to 32 \to 1$, Tanh, 5,000 epochs on $[-2, 2]$). Additionally, each operation is implemented as a hand-coded Python function for comparison. Nine composition chains of depth 2--6 are evaluated in-distribution and at $4\times$ extrapolation.

\begin{table}[!htbp]
\caption{Compositional generalization. Compiled and hand-coded chains both achieve \textbf{exactly zero MSE} at all depths and evaluation ranges. Neural chains can accumulate large errors, especially in degree-amplifying chains (up to $5.9 \times 10^9$). The compiled vs.\ hand-coded equivalence confirms that the compiler's composition mechanism introduces no error beyond what manual function composition achieves.}
\label{tab:composition}
\centering
\small
\resizebox{\columnwidth}{!}{%
\begin{tabular}{@{}lcrrrrr@{}}
\toprule
& & \multicolumn{2}{c}{MSE (in-dist.)} & \multicolumn{2}{c}{MSE ($4\times$ extrap.)} \\
\cmidrule(lr){3-4} \cmidrule(lr){5-6}
Chain & Depth & C/HC & Neural & C/HC & Neural \\
\midrule
square $\to$ add\_one & 2 & 0 & 1.2e+0 & 0 & 7.7e+2 \\
sin $\to$ square & 2 & 0 & 3.2e-1 & 0 & 9.0e-1 \\
square $\to$ add\_one $\to$ cube & 3 & 0 & 1.4e+3 & 0 & 5.9e+9 \\
exp $\to$ negate $\to$ add\_one & 3 & 0 & 2.8e+0 & 0 & 2.8e+5 \\
sin $\to$ square $\to$ add\_one $\to$ sqrt\_abs & 4 & 0 & 3.0e-1 & 0 & 2.3e-1 \\
square $\to$ double $\to$ sin $\to$ add\_one & 4 & 0 & 1.3e+0 & 0 & 1.7e+0 \\
sin $\to$ exp $\to$ negate $\to$ add\_one $\to$ square & 5 & 0 & 9.3e-1 & 0 & 1.1e+0 \\
square $\to$ add\_one $\to$ cube $\to$ negate $\to$ add\_one & 5 & 0 & 1.5e+3 & 0 & 5.9e+9 \\
negate $\to$ \ldots $\to$ sin $\to$ add\_one & 6 & 0 & 1.3e+0 & 0 & 1.8e+0 \\
\bottomrule
\end{tabular}}
\end{table}

\paragraph{Results (Table~\ref{tab:composition}).} Compiled and hand-coded chains both achieve exactly zero MSE at all depths and evaluation ranges. Neural chains accumulate errors ranging from $10^{-1}$ to $5.9 \times 10^9$ at $4\times$ extrapolation. The exponential blowup in polynomial-amplifying chains (square $\to$ add\_one $\to$ cube) confirms Proposition~\ref{prop:composition}: per-module errors are amplified by the Lipschitz constants of downstream modules.

This experiment crystallizes the compiler's core value: for \emph{individual} operations, compiled and hand-coded are identical. But for \emph{composition}, the compiler provides a systematic mechanism to chain modules from symbolic specifications, while hand-coding requires manual reimplementation of each new chain. Figure~\ref{fig:composition} visualizes the error amplification across chain depths and evaluation ranges.

\begin{figure}[!htbp]
\centering
\includegraphics[width=\textwidth]{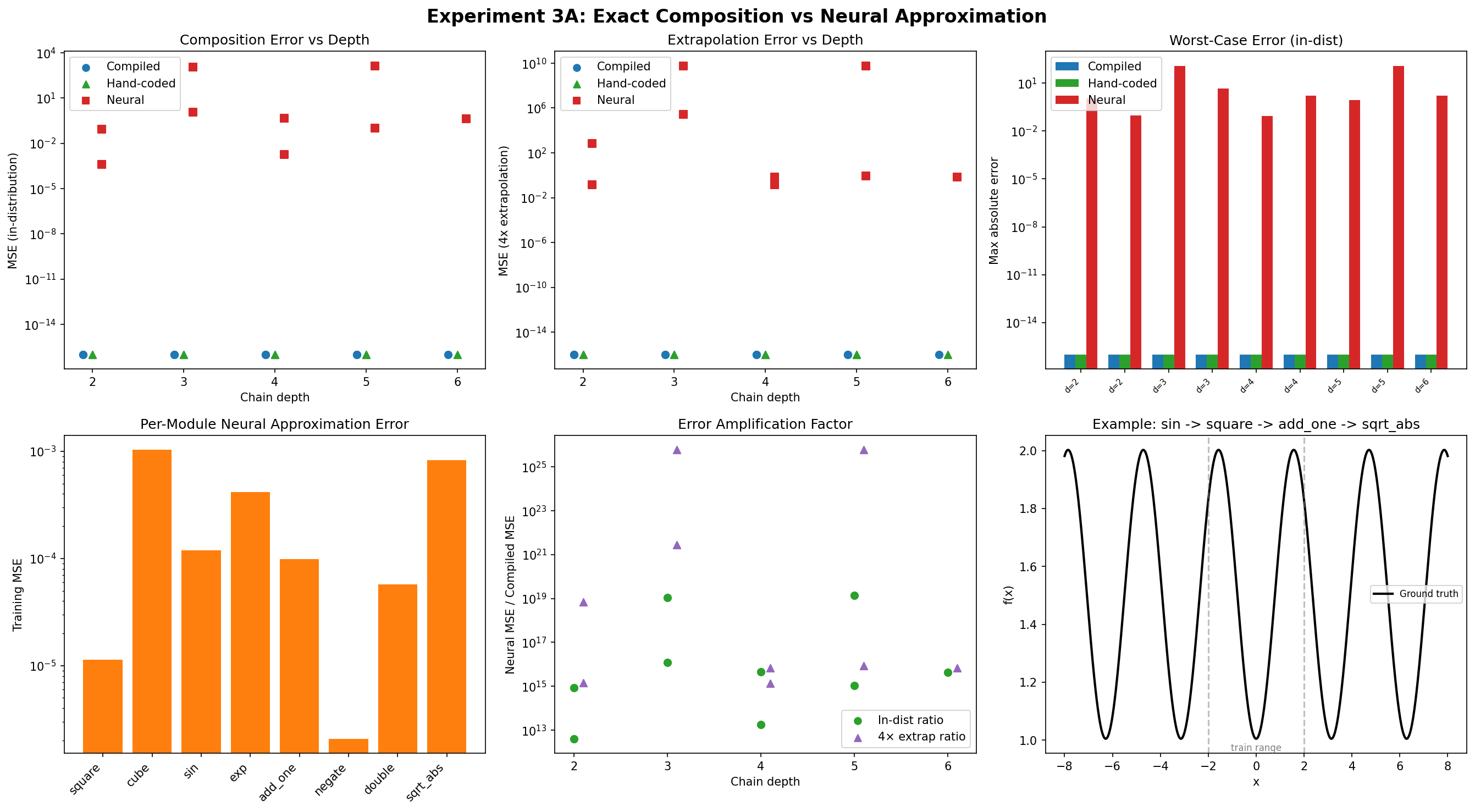}
\caption{Compositional generalization. \textbf{Top left}: In-distribution composition error vs.\ chain depth: compiled and hand-coded chains maintain zero error at all depths while neural chains accumulate errors. \textbf{Top center}: Extrapolation ($4\times$ range) error vs.\ depth: neural error amplifies dramatically with composition depth, particularly under extrapolation. \textbf{Top right}: Worst-case absolute error (in-distribution) per chain: compiled/hand-coded remain at machine epsilon. \textbf{Bottom left}: Per-module neural approximation error (training loss for each of the 8 operations). \textbf{Bottom center}: Error amplification factor (neural MSE / compiled MSE) for in-distribution and extrapolation. \textbf{Bottom right}: Example chain (sin $\to$ square $\to$ add\_one $\to$ sqrt\_abs) showing the function computed by the composition.}
\label{fig:composition}
\end{figure}

\section{Computational Cost}
\label{sec:cost}

We benchmark on an NVIDIA RTX 4090 (n128 node, University of Idaho HPC cluster). The \texttt{DirectModule} backend executes compiled instruction sequences directly, avoiding the overhead of graph framework dispatch.

\begin{table}[!htbp]
\caption{Computational cost summary. Compilation is a one-time cost under 150$\mu$s. The DirectModule adds negligible overhead compared to hand-coded PyTorch at batch size 1, and at batch size 10K the per-sample cost drops to nanoseconds. Training overhead for hybrid models is 2.5--3.8$\times$ vs.\ a pure MLP, modest relative to the accuracy gains.}
\label{tab:cost}
\centering
\small
\begin{tabular}{@{}lrr@{}}
\toprule
Metric & Value & Notes \\
\midrule
\multicolumn{3}{l}{\textbf{Compilation}} \\
\quad Compile time (all programs) & 30--150 $\mu$s & One-time cost \\
\quad Model initialization & 28--58 $\mu$s & One-time cost \\
\midrule
\multicolumn{3}{l}{\textbf{Single evaluation (batch=1)}} \\
\quad DirectModule CPU & 12--95 $\mu$s & Comparable to sequential eval \\
\quad Sequential (Python loop) & 16--78 $\mu$s & Interpreter overhead \\
\midrule
\multicolumn{3}{l}{\textbf{Batch throughput (batch=10K, CPU)}} \\
\quad Simple programs (3--7 nodes) & 585M--2.0B samples/s & Amortized fixed cost \\
\quad Complex programs (9--15 nodes) & 100M--460M samples/s & \\
\midrule
\multicolumn{3}{l}{\textbf{GPU throughput (batch=10K, RTX 4090)}} \\
\quad Simple programs & 135M--273M samples/s & \\
\quad Large trees (3,999 nodes) & 9.7M samples/s & 296$\times$ vs.\ Python \\
\quad Deep chains (501 nodes) & 333M samples/s & \\
\midrule
\multicolumn{3}{l}{\textbf{Training overhead vs MLP}} \\
\quad Known structure (2--4 params) & 3.8$\times$ & Subgraph dominates \\
\quad Hybrid (compiled + MLP) & 2.5$\times$ & MLP forward/backward dominates \\
\quad Neural ODE (torchdiffeq) & 2.9$\times$ & Adaptive solver overhead \\
\bottomrule
\end{tabular}
\end{table}

\paragraph{When is the overhead justified?} The pendulum achieves $731\times$ better MSE at $3.8\times$ wall-clock cost, an accuracy-per-compute ratio of $\sim 190\times$ in favor of compilation. In the tested setting, the MLP does not reach the compiled model's accuracy, because the advantage comes from exact structural specification rather than longer training. Training cost is also a one-time expense; at inference time, the frozen module requires no optimization but still incurs the forward-evaluation cost of the compiled instruction sequence.

\paragraph{Scaling.} For large programs (3,999 nodes), GPU batch evaluation achieves $296\times$ speedup over Python, confirming that the instruction-dispatch architecture scales to complex expressions. Deep chains (501 nodes, depth 250) sustain 333M samples/s on GPU. Figure~\ref{fig:param_recovery} summarizes parameter-recovery error across the experiments with PINN baselines.

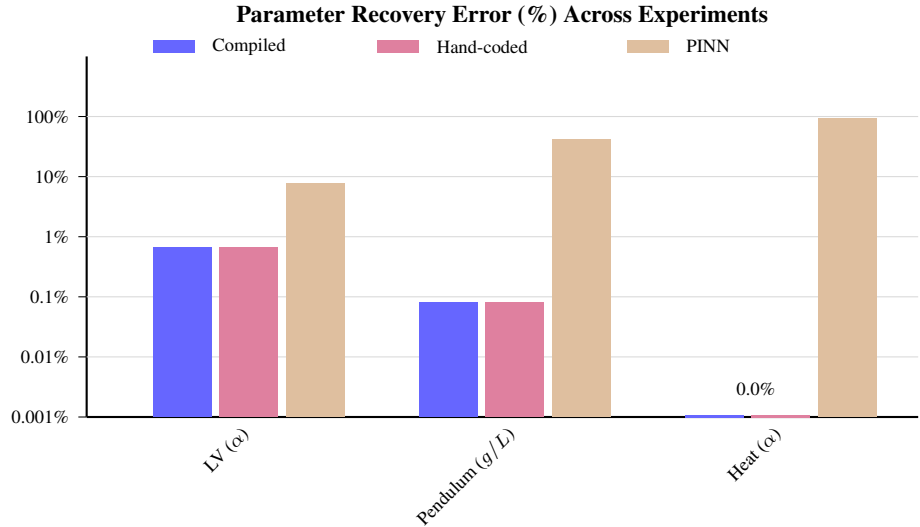
\begin{figure}[!htbp]
\centering
\begin{tikzpicture}[xscale=1.1, yscale=0.53]
    \node[font=\small\bfseries] at (5, 8.5) {Parameter Recovery Error (\%) Across Experiments};

    \fill[blue!60] (0.8,7.6) rectangle (1.3,7.9);
    \node[font=\scriptsize, right] at (1.4,7.75) {Compiled};
    \fill[purple!50] (3.5,7.6) rectangle (4.0,7.9);
    \node[font=\scriptsize, right] at (4.1,7.75) {Hand-coded};
    \fill[brown!50] (6.5,7.6) rectangle (7.0,7.9);
    \node[font=\scriptsize, right] at (7.1,7.75) {PINN};

    \draw[thick] (0,-1.5) -- (0,7.5);
    \foreach \y/\label in {-1.5/0.001, 0/0.01, 1.5/0.1, 3/1, 4.5/10, 6/100} {
        \draw (0,\y) -- (-0.1,\y) node[left, font=\scriptsize] {\label\%};
        \draw[gray!30] (0,\y) -- (10,\y);
    }

    \draw[thick] (0,-1.5) -- (10,-1.5);

    \fill[blue!60] (0.8,-1.5) rectangle (1.5,2.74);
    \fill[purple!50] (1.6,-1.5) rectangle (2.3,2.74);
    \fill[brown!50] (2.4,-1.5) rectangle (3.1,4.33);
    \node[font=\scriptsize, rotate=45, anchor=east] at (2.0,-1.7) {LV ($\alpha$)};

    \fill[blue!60] (4.0,-1.5) rectangle (4.7,1.35);
    \fill[purple!50] (4.8,-1.5) rectangle (5.5,1.35);
    \fill[brown!50] (5.6,-1.5) rectangle (6.3,5.42);
    \node[font=\scriptsize, rotate=45, anchor=east] at (5.2,-1.7) {Pendulum ($g/L$)};

    \fill[blue!60] (7.2,-1.5) rectangle (7.9,-1.45);
    \fill[purple!50] (8.0,-1.5) rectangle (8.7,-1.45);
    \fill[brown!50] (8.8,-1.5) rectangle (9.5,5.95);
    \node[font=\scriptsize] at (8.05,-0.8) {0.0\%};
    \node[font=\scriptsize, rotate=45, anchor=east] at (8.4,-1.7) {Heat ($\alpha$)};
\end{tikzpicture}
\caption{Parameter recovery error comparison for experiments with PINN baselines. Compiled and hand-coded models (hard constraints) produce identical results in every case. PINNs (soft constraints) show 7--93\% error where compiled models achieve $<$1\%. The gap is most dramatic for the heat equation (0.0\% vs.\ 92.7\%) and pendulum (0.08\% vs.\ 41.1\%). Note log scale.}
\label{fig:param_recovery}
\end{figure}

\section{The Compiler's Value: A Nuanced Assessment}
\label{sec:value}

Our experiments reveal a precise characterization of the compiler's contribution.

\paragraph{What the compiler does NOT provide.} For any single, fixed equation, the compiled module produces output numerically identical to a hand-coded PyTorch implementation (Tables~\ref{tab:feynman}--\ref{tab:pde}). If a practitioner has one equation and is willing to implement it manually, the compiler offers no accuracy benefit. This is by design: the compiler translates programs faithfully (Theorem~\ref{thm:correctness}), not approximately.

\paragraph{What the compiler DOES provide.}

\begin{enumerate}[leftmargin=2em,topsep=2pt,itemsep=3pt]
    \item \textbf{Systematic module generation.} A single function call, \texttt{compile\_scheme(source, inputs=\allowbreak\{...\})}, produces a correct, differentiable, composable \texttt{nn.Module} from any expression in the 51-operation language. Minimal manual implementation, reduced per-equation debugging, and no boilerplate.

    \item \textbf{Composability from text.} Changing the string \texttt{"(sin x)"} to \texttt{"(exp (sin x))"} produces a new correct module instantly. Hand-coding each composition requires writing, testing, and debugging new PyTorch code. For a library of $n$ modules composed in chains of depth $k$, hand-coding requires $O(n^k)$ implementations; the compiler requires $O(n)$ source strings.

    \item \textbf{Hard constraints vs.\ soft constraints.} Both compiled and hand-coded models enforce physics exactly. PINNs enforce physics as a soft loss penalty. Our experiments show this distinction matters dramatically: hard constraints achieve 0.0\% parameter recovery error where soft constraints show 7--93\% error (Tables~\ref{tab:lv}--\ref{tab:pde}). The compiler provides a systematic mechanism to generate hard-constrained modules from specifications; the PINN provides only soft constraints.

    \item \textbf{A programmatic interface for model generation.} The compiler's input is a string. This makes it amenable to programmatic generation: scripts can generate families of equations, optimizers can search over equation structures, and, most promisingly, large language models can translate natural-language physics descriptions into compilable specifications (Section~\ref{sec:future}).
\end{enumerate}

\section{Related Work}
\label{sec:related}

\paragraph{Physics-informed neural networks.} PINNs~\cite{raissi2019physics,hao2023pinnacle} encode physics as soft loss penalties via autograd. The network may violate the physics; our compiled modules enforce it exactly. Our experiments confirm this distinction: PINNs show 7--93\% parameter recovery error where compiled models achieve ${<}\,1$\% (Tables~\ref{tab:lv}--\ref{tab:pde}). PINNs also struggle with stiff systems and require careful loss weighting~\cite{wang2021understanding}.\looseness=-1

\paragraph{Universal differential equations.} Rackauckas et al.~\cite{rackauckas2020universal} combine known ODE terms with neural networks in Julia's DifferentialEquations.jl. This is closest to our hybrid architecture (Pattern 2) but requires hand-implementation of each known term in Julia. Our compiler automates this from symbolic specifications.

\paragraph{Neural ODEs.} Chen et al.~\cite{chen2018neural} parameterize the full ODE right-hand side with a neural network. This discards known structure. Our approach compiles known terms and learns only the unknown remainder, achieving $731\times$ better accuracy with $4{,}353\times$ fewer parameters on the pendulum.

\paragraph{Hamiltonian and Lagrangian neural networks.} HNN~\cite{greydanus2019hamiltonian} and LNN~\cite{cranmer2020lagrangian} enforce conservation laws via architectural constraints. These are complementary: they encode \emph{structural} knowledge (Hamiltonian/Lagrangian structure); our compiler encodes \emph{specific} known equations. The two could be combined, compiling known Hamiltonian terms while learning unknown contributions.\looseness=-1

\paragraph{Sparse identification (SINDy).} Brunton et al.~\cite{brunton2016discovering} discover ODE structure from data via sparse regression on a candidate library. SINDy discovers structure; our compiler exploits known structure. They are complementary: SINDy could identify candidate equations that the compiler then integrates into hybrid models.

\paragraph{Symbolic regression.} AI Feynman~\cite{udrescu2020ai} and PySR~\cite{cranmer2023interpretable} discover equation forms from data. When the form is known, compilation is dramatically more efficient: 1--3 parameters vs.\ thousands of function evaluations. The two approaches address different problems (discovery vs.\ exploitation).

\paragraph{Differentiable programming and symbolic code generation.} JAX~\cite{jax2018github}, Julia/Zygote~\cite{innes2019differentiable}, TorchScript, and SymPy-to-PyTorch pipelines (e.g., \texttt{lambdify}) compile or translate symbolic expressions to differentiable code. These are faster than our approach for single evaluation. Our contribution is not differentiable compilation per se, but the \emph{frozen composable module abstraction}: compiled programs become \texttt{nn.Module} objects that can be composed into chains with zero error (Section~\ref{sec:composition}), embedded in hybrid architectures with trainable components, and programmatically generated from symbolic specifications (Section~\ref{sec:future}). The key difference from ``SymPy to PyTorch'' is the systematic evaluation framework: we compare against hand-coded, PINN, MLP, and neural ODE baselines across six domains, establishing when hard constraints outperform soft constraints experimentally.

\paragraph{Neural module networks.} Andreas et al.~\cite{andreas2016neural} compose task-specific neural modules selected by a controller. Our modules are \emph{compiled} (not learned), providing exactness guarantees that learned modules cannot.

\paragraph{Neurosymbolic programming.} Scallop~\cite{li2023scallop}, DeepProbLog~\cite{manhaeve2018deepproblog}, and NeurASP~\cite{yang2020neurasp} integrate neural perception with symbolic reasoning via differentiable logic programming. These systems learn to map perceptual inputs to symbolic predicates; our compiler addresses a different problem: translating known mathematical expressions into frozen differentiable modules for embedding in hybrid scientific models.

\paragraph{Compilation to neural architectures.} Tracr~\cite{lindner2023tracr} compiles RASP programs to transformer weights. We target a different domain: exact arithmetic computation embedded in trainable scientific models, rather than interpretable transformer circuits.

\section{Future Directions: LLM-Driven Scientific Model Construction}
\label{sec:future}

The compiler's input is a text string; its output is a differentiable module. This creates a natural interface for large language models.

\paragraph{Natural-language to compiled module.} A scientist describes ``the gravitational force between two point masses'' in English. An LLM translates this to \texttt{(/ (* (- 0 G) (* m1 m2)) (pow r 2))}, which the compiler converts to a frozen, differentiable module in milliseconds. The module is correct by construction (Theorem~\ref{thm:correctness}), provides exact gradients (Theorem~\ref{thm:gradients}), and can be composed with other modules (Proposition~\ref{prop:composition}). The LLM serves as a natural-language parser; the compiler provides formal guarantees that the LLM alone cannot.

\paragraph{Self-architecting models.} Consider a model that can modify its own architecture by generating Scheme programs and compiling them into new frozen modules. Each compiled module is:
\begin{itemize}[leftmargin=2em,topsep=2pt,itemsep=1pt]
    \item \emph{Differentiable}: gradients flow through the module back to trainable components.
    \item \emph{Interpretable}: each module is a readable program, not an opaque weight matrix.
    \item \emph{Verifiable}: the module can be tested against known inputs before deployment.
    \item \emph{Non-destructive}: frozen modules cannot interfere with existing capabilities.
\end{itemize}
This is a structured alternative to uncontrolled weight editing: the model grows new exact-computation capabilities through a constrained, verifiable channel.

\paragraph{Continual learning via compilation.} The model accumulates a library of compiled modules over time, each permanently frozen and exactly correct. New modules can compose with previous ones (Section~\ref{sec:composition}). Only the routing/interfacing layer requires updating. This is a structural solution to catastrophic forgetting: compiled capabilities never degrade because they are frozen. The composition guarantee (Proposition~\ref{prop:composition}) ensures that combining old and new modules introduces no additional error.

\paragraph{Iterative hypothesis refinement.} An LLM proposes a physics hypothesis as a program; the compiler integrates it into a trainable model; gradient descent evaluates the hypothesis against data; the residual informs a new hypothesis. This compile-train-refine loop could automate the scientific modeling cycle, with the compiler ensuring that each hypothesis is tested in its exact form rather than through a lossy approximation.

\section{Limitations}
\label{sec:limitations}

\begin{enumerate}[leftmargin=2em,topsep=2pt,itemsep=3pt]
    \item \textbf{The source language is restricted.} $\mathcal{L}$ supports 51 operations including vector and matrix algebra, but excludes higher-order functions, string processing, and general data structures. This covers a broad range of scientific equations but not all computational patterns.

    \item \textbf{Credit assignment in hybrid models.} When compiled and learned components have overlapping functional forms (Section~\ref{sec:pendulum}, Scenario 2), the optimizer may distribute the computation non-uniquely. This is a general limitation of additive hybrid architectures, not specific to our compiler, but it affects parameter interpretability.

    \item \textbf{The program must be known.} The compiler exploits known structure; it does not discover it. When the equation form is unknown, symbolic regression or SINDy methods are more appropriate. The compiler is complementary to these discovery methods.

    \item \textbf{Training overhead.} Compiled hybrid models train 2.5--3.8$\times$ slower than pure MLPs due to the instruction-dispatch evaluation of compiled subgraphs. For simple expressions where a hand-coded PyTorch implementation is trivial, the compiler adds engineering convenience but not computational value.

    \item \textbf{Periodic and singular equations.} Two of 15 Feynman equations fail to recover constants: the harmonic oscillator (periodic loss landscape) and the Lorentz factor (singularity). These are optimization failures, not compilation failures; the compiled module computes correctly, but gradient descent gets trapped.
\end{enumerate}

\section{Conclusion}
\label{sec:conclusion}

We have presented the Neural Compiler, a system that translates symbolic programs into frozen, differentiable PyTorch modules for hybrid scientific machine learning. The compiler's theoretical guarantees (compilation correctness, gradient exactness, compiled-component exactness, and exact composition) are confirmed across six experiment domains spanning algebraic equations, ODEs, a PDE, vector mechanics, and compositional chains.

The key empirical finding is a clean separation of concerns: compiled and hand-coded models produce numerically identical results for single equations (confirming zero numerical discrepancy), while the gap between hard-constrained (compiled/hand-coded) and soft-constrained (PINN) approaches is dramatic: 0\% vs.\ 93\% parameter recovery error on the heat equation, 0.08\% vs.\ 41\% on the pendulum. The compiler's value is \emph{systematic composability}: generating correct modules from text specifications rather than manual implementation, providing the same exactness as hand-coded composition while making large families of compositions programmatically generable from symbolic specifications.

The system's string-in, module-out interface makes it a natural target for large language model integration, opening a path toward self-architecting scientific models that grow new exact-computation capabilities through compilation. The source code is available at \url{https://github.com/sheneman/neural_compiler}.


\bibliographystyle{unsrtnat}

\newpage
\appendix

\section{Full Feynman Equation Specifications}
\label{app:feynman}

\begin{table}[!htbp]
\caption{Feynman equations: Scheme source and compiled module statistics.}
\label{tab:feynman_full}
\centering
\footnotesize
\begin{tabular}{@{}llrrl@{}}
\toprule
Equation & Scheme Source & N & D & Trainable \\
\midrule
Planck & \texttt{(* h f)} & 3 & 1 & $h$ \\
Hooke & \texttt{(* (- 0 k) x)} & 5 & 2 & $k$ \\
KE & \texttt{(* alpha (* m (pow v 2)))} & 6 & 3 & $\alpha$ \\
Gravity & \texttt{(/ (* G (* m1 m2)) (pow r 2))} & 8 & 3 & $G$ \\
Ideal gas & \texttt{(* n (* R T))} & 5 & 2 & $R$ \\
Pendulum & \texttt{(* k (sqrt (/ L g)))} & 6 & 3 & $k$ \\
Heat & \texttt{(* m (* c dT))} & 5 & 2 & $c$ \\
Coulomb & \texttt{(/ (* ke (* q1 q2)) (pow r 2))} & 8 & 3 & $k_e$ \\
Gaussian & \texttt{(/ (exp ...) (* sigma (sqrt ...)))} & 15 & 5 & $\mu, \sigma$ \\
Rel.\ energy & \texttt{(/ (* m (pow c 2)) (sqrt ...))} & 14 & 5 & $m, c$ \\
Sound & \texttt{(sqrt (/ (* gamma P) rho))} & 7 & 3 & $\gamma$ \\
Barometric & \texttt{(* P0 (exp (/ ...)))} & 14 & 5 & $P_0, m, k_B$ \\
E-field & \texttt{(* coeff (* E (* E V)))} & 6 & 3 & $\epsilon$ \\
Oscillator & \texttt{(* A (sin (+ (* omega t) phi)))} & 8 & 4 & $A, \omega, \phi$ \\
Lorentz & \texttt{(/ 1 (sqrt (- 1 (pow ...))))} & 10 & 5 & $c$ \\
\bottomrule
\end{tabular}
\end{table}

\section{Noise Robustness (Lotka-Volterra)}
\label{app:noise}

\begin{table}[!htbp]
\caption{Parameter recovery error vs.\ observation noise (Lotka-Volterra, known structure, 3{,}000 epochs). At 0\% noise, all parameters are recovered exactly. Recovery degrades gracefully up to 10\% noise.}
\label{tab:noise}
\centering
\begin{tabular}{@{}rrrrrr@{}}
\toprule
Noise & $\alpha$ err & $\beta$ err & $\delta$ err & $\gamma$ err & Max err \\
\midrule
0\% & 0.000\% & 0.000\% & 0.000\% & 0.000\% & 0.000\% \\
1\% & 0.663\% & 0.502\% & 0.198\% & 0.014\% & 0.663\% \\
2\% & 1.166\% & 0.945\% & 0.643\% & 0.140\% & 1.166\% \\
5\% & 1.923\% & 1.612\% & 3.232\% & 1.942\% & 3.232\% \\
10\% & 0.550\% & 0.960\% & 10.846\% & 8.772\% & 10.846\% \\
\bottomrule
\end{tabular}
\end{table}

\section{Batch Throughput Scaling}
\label{app:throughput}

\begin{table}[!htbp]
\caption{Batch throughput (samples/second) on RTX 4090 for representative programs. The fixed per-call overhead is amortized across batch elements, yielding near-linear scaling.}
\label{tab:throughput}
\centering
\small
\begin{tabular}{@{}lrrrr@{}}
\toprule
Program (nodes) & Batch 1 & Batch 100 & Batch 10K & Scaling \\
\midrule
add (3) & 201K & 11.0M & 2.0B & 9{,}950$\times$ \\
square\_plus (4) & 202K & 8.2M & 585M & 2{,}900$\times$ \\
four\_ops (7) & 141K & 5.9M & 878M & 6{,}200$\times$ \\
quadratic (9) & 115K & 4.5M & 365M & 3{,}200$\times$ \\
discriminant (13) & 28K & 2.0M & 135M & 4{,}800$\times$ \\
dot4 (15) & 85K & 3.6M & 457M & 5{,}400$\times$ \\
\bottomrule
\end{tabular}
\end{table}

\section{Gradient Scaling in Deep Compiled Chains}
\label{app:gradients}

\begin{proposition}[Gradient Scaling]
For a $k$-stage pipeline of polynomial operations $G = g_k \circ \cdots \circ g_1$ where each $g_i$ has degree $d_i$, the gradient magnitude satisfies $|\partial G / \partial x| = \prod_{i=1}^{k} d_i |z_{i-1}|^{d_i-1} |c_i|$ where $z_i$ are intermediate values. For squaring operations ($d_i=2$): $|\partial G / \partial x| = 2^k \prod |z_i|$, growing exponentially when $|z_i| > 1$.
\end{proposition}

This gradient scaling is mathematically exact; the compiled module faithfully reproduces it because the computation is exact (Theorem~\ref{thm:correctness}). Residual connections at module interfaces ($z_i + g_i(z_i)$) bypass the multiplicative chain, providing an alternative gradient pathway that resolves gradient traps. In our experiments, residual connections achieve 100\% convergence across random seeds for deep compiled chains.

\section{Vector and Matrix Operations}
\label{app:vector}

Table~\ref{tab:vector_ops} lists the vector and matrix operations added in v0.8.0. All operations use \texttt{dim=-1} / \texttt{dim=-2} conventions for batch compatibility.

\begin{table}[!htbp]
\caption{Vector and matrix operations (selected). All use \texttt{dim=-1}/\texttt{dim=-2} for batch compatibility.}
\label{tab:vector_ops}
\centering
\footnotesize
\begin{tabular}{@{}lllll@{}}
\toprule
Op & Type & Implementation & In & Out \\
\midrule
\texttt{vec} & construct & \texttt{torch.stack(args, dim=-1)} & $[*]^n$ & $[*, n]$ \\
\texttt{dot} & reduce & \texttt{(a*b).sum(dim=-1)} & $[*, n]^2$ & $[*]$ \\
\texttt{cross} & binary & \texttt{torch.linalg.cross} & $[*, 3]^2$ & $[*, 3]$ \\
\texttt{norm} & reduce & \texttt{torch.norm(v, dim=-1)} & $[*, n]$ & $[*]$ \\
\texttt{matvec} & binary & \texttt{torch.matmul(M, v)} & $[*, n, m] \times [*, m]$ & $[*, n]$ \\
\texttt{matmul} & binary & \texttt{torch.matmul(A, B)} & $[*, n, k] \times [*, k, m]$ & $[*, n, m]$ \\
\texttt{det} & reduce & \texttt{torch.linalg.det} & $[*, n, n]$ & $[*]$ \\
\texttt{inv} & unary & \texttt{torch.linalg.inv} & $[*, n, n]$ & $[*, n, n]$ \\
\texttt{outer} & binary & \texttt{a[...,None]*b[...,None,:]} & $[*, n] \times [*, m]$ & $[*, n, m]$ \\
\bottomrule
\end{tabular}
\end{table}

\end{document}